\documentclass[11pt,fullpage,letterpaper]{article} 

\usepackage{natbib}
\setcitestyle{round,citesep={;},aysep={,},yysep={;}}
\usepackage[english]{babel}

\usepackage[margin=1in]{geometry}
\usepackage{fancyhdr}
\usepackage[utf8]{inputenc} 
\usepackage[T1]{fontenc}    

\usepackage{titletoc}
\usepackage[toc, page, header]{appendix} 

\usepackage{url}            
\usepackage{booktabs}       
\usepackage{amsfonts}       
\usepackage{nicefrac}       
\usepackage{microtype}      
\usepackage{xcolor}         

\usepackage{parskip}

\title{Policy Regret for Embedding Model Routing: \\ Contextual Bandits with Low-Rank Experts}
\author{%
    Yan Dai\thanks{Operations Research Center, MIT. Email: \texttt{yandai20@mit.edu}.}
    \and
    Negin Golrezaei\thanks{Sloan School of Management, MIT. Email: \texttt{golrezae@mit.edu}.}
    \and
    Patrick Jaillet\thanks{Department of EECS, MIT. Email: \texttt{jaillet@mit.edu}.}
}
\date{}

\usepackage{multirow}
\usepackage{arydshln}
\usepackage{footnote}
\usepackage{enumitem}
\setitemize{leftmargin=*, nosep, itemsep=2pt, parsep=3pt}
\setenumerate{leftmargin=*, nosep, itemsep=2pt, parsep=3pt}

\usepackage{amsmath}
\usepackage{amsthm}
\usepackage{amssymb}
\usepackage{mathtools}
\usepackage{thmtools}
\usepackage{comment}
\usepackage{bbm}
\usepackage{bm}
\allowdisplaybreaks

\let\originalmiddle=\middle
\def\middle#1{\mathrel{}\originalmiddle#1\mathrel{}}

\usepackage{algorithm}
\usepackage[noEnd,commentColor=blue!70!white]{algpseudocodex}

\makeatletter
\newcommand{\algeq}[2][]{%
  \par%
  \noindent%
  \makebox[\columnwidth]{%
    \hfill #2%
    \if\relax\detokenize{#1}\relax
      \hfill
    \else
      \hfill (\refstepcounter{equation}\theequation)\label{#1}%
    \fi
  }%
  \par%
}
\makeatother

\usepackage[colorlinks=true, linkcolor=blue!60!black, citecolor=blue!60!black,
urlcolor=black]{hyperref}
\usepackage{cleveref}
\AddToHook{cmd/appendix/before}{%
    \crefalias{section}{appendix}%
    \crefalias{subsection}{appendix}
}

\Crefname{ALG@line}{Line}{Lines}
\Crefname{assumption}{Assumption}{Assumptions}
\Crefname{protocol}{Protocol}{Protocols}
\Crefformat{equation}{Eq. #2(#1)#3}
\Crefrangeformat{equation}{Eqs. #3(#1)#4 to #5(#2)#6}
\Crefmultiformat{equation}{Eqs. #2(#1)#3}{ and #2(#1)#3}{, #2(#1)#3}{ and #2(#1)#3}
\Crefrangemultiformat{equation}{Eqs. #3(#1)#4 to #5(#2)#6}{ and #3(#1)#4 to #5(#2)#6}{, #3(#1)#4 to #5(#2)#6}{ and #3(#1)#4 to #5(#2)#6}

\newtheorem{theorem}{Theorem}
\newtheorem{lemma}[theorem]{Lemma}
\newtheorem{proposition}[theorem]{Proposition}

\theoremstyle{definition}
\newtheorem{definition}{Definition}

\usepackage{tikz}
\usetikzlibrary{positioning}

\newcommand{\E}{\operatornamewithlimits{\mathbb{E}}}
\newcommand{\argmax}{\operatornamewithlimits{\mathrm{argmax}}}
\newcommand{\argmin}{\operatornamewithlimits{\mathrm{argmin}}}
\newcommand{\arcsinh}{\operatorname{\mathrm{arcsinh}}}
\newcommand{\diag}{\operatorname{\mathrm{diag}}}

\newcommand{\tr}{\operatornamewithlimits{\mathrm{Tr}}}
\newcommand{\rank}{\operatornamewithlimits{\mathrm{rank}}}

\renewcommand{\O}{\operatorname{\mathcal O}}
\newcommand{\Otil}{\operatorname{\tilde{\O}}}
\newcommand{\trans}{\mathsf{T}}
\newcommand{\mA}{\mathcal{A}}
\newcommand{\mQ}{\mathcal{Q}}
\newcommand{\mW}{\mathbb{W}}
\newcommand{\mR}{\mathfrak{R}}

\renewcommand{\tilde}{\widetilde}
\renewcommand{\hat}{\widehat}
\renewcommand{\bar}{\overline}

\renewcommand{\paragraph}[1]{\vspace{2pt}\noindent\textbf{#1}}

\usepackage{afterpage}
\usepackage{wrapfig}
\usepackage{pifont}
\newcommand{\cmark}{{\color{green!50!black}\ding{51}}}%
\newcommand{\xmark}{{\color{red!50!black}\ding{55}}}%

\begin{document}

\maketitle

\begin{abstract}
\noindent Modern recommendation systems increasingly rely on dynamically routing diverse queries to multiple embedding models. Despite its practical significance, this problem remains poorly understood under realistic conditions like adversarial queries, bandit feedback, and limited observability of models.  We formalize embedding model routing as an adversarial contextual linear bandit with low-rank experts, where contexts are queries, actions are items, and experts are the embedding models working on low-rank latent representation spaces.
We first establish that standard regret notions suffer from structural misspecification or statistical intractability, and we identify a \emph{log-quadratic} policy class that is expressive enough to capture query-dependent model routing, yet structured enough to allow efficient online learning.
Focusing on this log-quadratic policy optimization problem under bandit feedback -- which is of independent interest -- we propose a policy gradient algorithm called Hypentropy Policy Gradient (HPG). It provably adapts to the unknown low-rank structure under incomplete information and attains $\widetilde{\mathcal O}(s\sqrt{M T})$ linearized policy regret -- where $s, M$, and $T$ are the intrinsic rank of the experts, the number of models, and the number of rounds -- thus avoiding a curse of dimensionality. We provide computationally efficient and parameter-free implementation of HPG.
\end{abstract}

\section{Introduction}

\begin{wrapfigure}{r}{0.3\textwidth}
\vspace{-10pt}
\includegraphics[width=\linewidth]{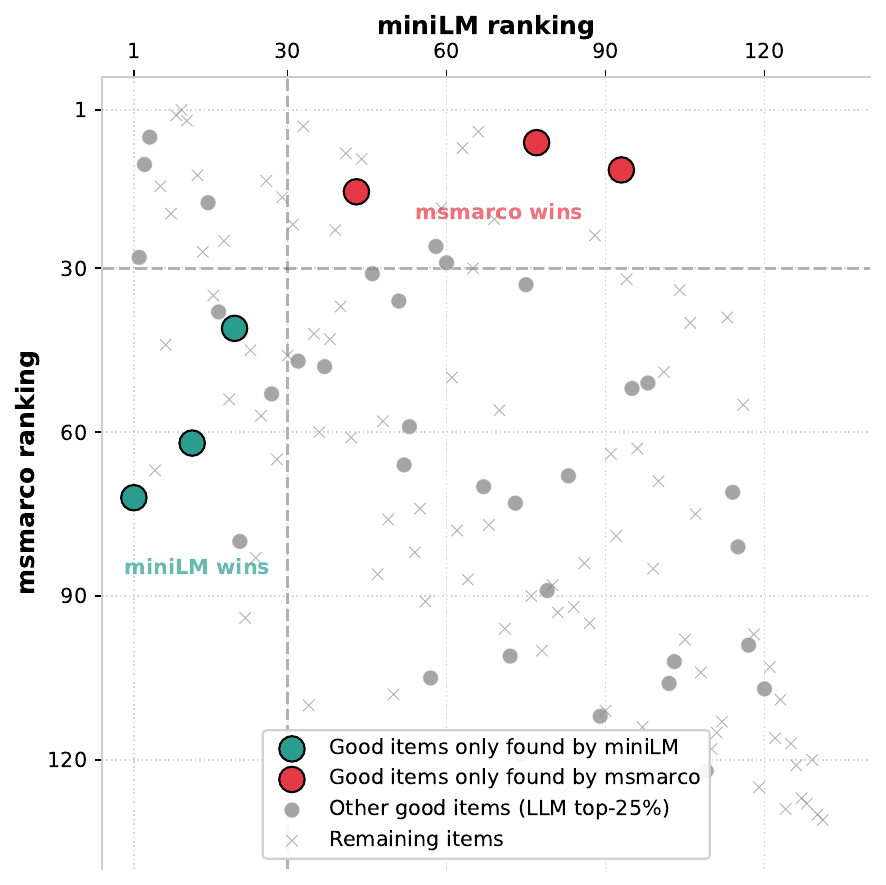}
\caption{Rankings of two different embedding models}
\label{fig:gemini_v_embeddings}
\vspace{-10pt}
\end{wrapfigure}

Modern recommendation systems and search engines rely heavily on dense embedding models for information retrieval \citep{yi2019sampling,karpukhin2020dense}. These models map queries and items into low-dimensional ``embedding spaces,'' enabling efficient retrieval based on the similarity scores (i.e., inner products). However, representing the full diversity of distinct user queries within a single embedding space is fundamentally limited: Due to models' intrinsic low-dimensional structures, no single embedding model can be uniformly optimal across all types of queries \citep{weller2025theoretical}.

To illustrate, \Cref{fig:gemini_v_embeddings} ranks items from the Amazon Berkeley Objects dataset \citep{collins2022abo} by a semantic embedding \texttt{miniLM} and by a Q\&A model \texttt{msmarco} against an LLM-generated ground truth. Gray dots are high-quality items (LLM top-25\%) and colored dots are those one model retrieves in its top-30 while the other misses (see \Cref{sec:experiment} for more details on the experimental setup). The errors and strengths are mutually exclusive, meaning \emph{no single model dominates the other}, and some adaptive routing is essential.

This limitation is reflected in practice. Large-scale industrial e-commerce platforms by Alibaba and Meta use multiple embedding models to capture various semantic aspects of queries \citep{li2019multi,huang2020embedding}. Recent approaches in Retrieval-Augmented Generation (RAG) also route queries to specialized ``expert models'' trained on domain-specific data \citep{lee2025routerretriever,zhao2026r}.
Nevertheless, most of such adaptive routing approaches are based on heuristics, which gives our main question: \textbf{\textit{How can an online learner adaptively route queries to embedding models, using feedback only from the routed models, to achieve provable performance guarantees?}}
Towards a theoretically grounded answer, this paper makes the following main contributions.

\paragraph{Modeling (\Cref{sec:setup}).}
We frame the problem of routing queries to embedding models as a $T$-round adversarial contextual linear bandit with low-rank experts: In round $t$, the learner observes an adversarial query $q_t\in \mathbb{R}^{d_q}$ as \emph{context} and faces an item set $\mA\subseteq \mathbb R^{d_a}$ as \emph{actions}. $M$ embedding models, as \emph{experts}, each implicitly has a distribution over $\mA$ via a latent embedding space of dimension $s\ll \min(d_q,d_a)$. The learner can only choose \emph{one} model $m_t$ to route, observe the selected model's recommendation, and receive an induced noisy reward. The goal is to compete with a fixed class of routing policies. This formulation captures the main challenges in realistic recommendation systems: adversarial query sequences, bandit reward feedback, and partial observability of model outputs.

\paragraph{Regret Notion (\Cref{sec:regret}).} A key challenge is identifying an appropriate performance benchmark for this problem. We find standard regret notions inadequate: Competing with the best fixed model (i.e., expert regret) ignores query heterogeneity, while competing with arbitrary query-dependent policies (i.e., contextual regret) is intractable. The commonly used log-linear policy class also fails to fit the underlying structures of embedding models.
Instead, our main insight is, under the latent geometry of embedding models, the expected reward of each model is approximately a \emph{low-rank quadratic} function of the query (\Cref{thm:quadratic policy class}). This leads to our log-quadratic policy class: expressive enough to capture near-optimal routing, while remaining amenable to efficient learning. Our claim is numerically justified on an Amazon ESCI dataset \citep{reddy2022shopping} in \Cref{tab:esci_direct_label_gaps}.

\paragraph{Algorithm and Implementation (\Cref{sec:OMD,sec:practical implementation}).}
For log-quadratic policy optimization under bandit feedback -- a problem of independent interest -- the quadratically large parameter space makes a vanilla policy gradient suffer from $\Otil(d_q\sqrt{MT})$ linearized policy regret (a surrogate metric to resolve the non-convexity of softmax; see \Cref{sec:parameter space}). To avoid the $d_q$ dependency, we propose Hypentropy Policy Gradient (HPG), which automatically adapts to the unknown low-rank structure of embedding models and provably attains $\Otil(s\sqrt{MT})$ linearized policy regret. HPG builds on Online Mirror Descent, whose Bregman projection is typically expensive; combining hypentropy regularization with a nuclear-norm parameter space brings it down to $\O(d_q^2M)$ time per round (\Cref{lem:efficient implementation}), and HPG can further be implemented parameter-freely, i.e., without knowing $s$.

\subsection{Related Literature}\label{sec:related work}
\paragraph{Model Routing.}
In large-scale industrial recommendation systems or search engines, multi-embedding frameworks were developed for better performance \citep{li2019multi,huang2020embedding}.
Mixture-of-Experts (MoE) architectures were proposed in the natural language processing literature to aggregate representations and capture semantic intents \citep{shazeer2017outrageously,ma2018modeling,lepikhin2021gshard,fedus2022switch,zhou2022mixture}.
In the era of large language models (LLMs), dynamic model routing (also known as model cascading or LLM routing) was proposed to balance the generation quality and the inference cost \citep{jiang2023llm,shnitzer2023large,chen2024frugalgpt,ong2025routellm}.
More closely related to our context of embedding models, modern Retrieval-Augmented Generation (RAG) pipelines increasingly route queries to specialized expert retrievers \citep{mallen2023not,jeong2024adaptive,lee2025routerretriever,zhao2026r}.
Several prior works also drew heuristics by viewing model routing as (online or offline) contextual bandits \citep{nguyen2024metallm,ong2025routellm,hu2025an,tsiourvas2025causal,jitkrittum2026universal,zu2026barouter,poon2026multillm}, but they treated the underlying embedding models as complete black-boxes.
On the other hand, in this paper, we model the embedding model routing rigorously as an adversarial contextual linear bandit with low-rank experts, and design a policy gradient algorithm with provable regret guarantee, computational cost bound, and parameter-free implementation.

\paragraph{Bandits with Expert Advice.}
Bandits with expert advice (\Cref{def:expert regret}) were first studied by \citet{auer2002nonstochastic}, who proposed the EXP4 algorithm for $\O(\sqrt{KT \log M})$ expert regret where $K,T,M$ are the numbers of arms, rounds, and experts. \citet{beygelzimer2011contextual} proposed EXP4.P, achieving the same bound with high probability. \citet{seldin2013open} raised the limited observation setup -- where each round only $N$ experts can be queried -- as a COLT open problem. With full information, $\Otil(\sqrt{MT/N})$ expert regret is possible \citep{seldin2014prediction}; under bandit feedback, \citet{kale2014multiarmed} gave $\Otil(\sqrt{\min(N,K)MT/N})$ expert regret and a near-matching lower bound.
The unlimited observation setup has recently regained interest, with $\Omega(\sqrt{KT \log(M/K)})$ lower bounds established for learners with different levels of adaptivity \citep{ito2024minimax,cesa2025improved,chase2025tight}.

\paragraph{Linear Bandits.}
Linear bandits were studied by \citet{dani2008stochastic}, and the contextual version (\Cref{def:contextual linear bandit}) was studied by \citet{chu2011contextual} and \citet{abbasi2011improved} under the adversarial-context stochastic-reward (``A-S'') regime.
The adversarial-context and adversarial-reward (``A-A'') regime is intractable \citep{kanade2014learning,hazan2016computational,neu2020efficient}, thus \citet{agarwal2014taming} and \citet{syrgkanis2016efficient} opted for policy regret.

The stochastic-context adversarial-reward (``S-A'') regime was proposed by \citet{neu2020efficient} under the known context distribution assumption. Subsequent works \citep{neu2021online,luo2021policy,dai2023refined,sherman2023improved,kong2024improved} focused on the more general linear Markov decision process (MDP) setup; when specialized to bandits, they require extra assumptions or give sub-optimal regret. \citet{liu2023bypassing} proposed an efficient $\Otil(d^2\sqrt T)$ algorithm and an inefficient $\Otil(d\sqrt T)$ algorithm, with $d$ being the dimension and $T$ being the number of rounds. \citet{ito2024minimax} considered a finite-arm case and \citet{van2025improved} refined the computational cost.

Misspecification in linear bandits was studied by, for example, \citet{gopalan2016low,ghosh2017misspecified,foster2020adapting,dong2023does,liu2024corruption}, which we adopt in \Cref{eq:true reward,eq:expert reward}. Richer non-linear regimes were also studied in the literature \citep{russo2013eluder,foster2018practical,foster2020beyond,li2022understanding}, which are beyond the scope of this paper.

\paragraph{Policy Gradient Methods.} Policy gradient (PG) methods are foundational for reinforcement learning and control \citep{williams1992simple,sutton1999policy,kakade2001natural}. Global convergence of PG methods for the log-linear policy class was initiated by \citet{fazel2018global} in control theory and \citet{bhandari2024global} in reinforcement learning. \citet{agarwal2021theory} and \citet{mei2020global} proposed gradient dominance conditions to establish exact convergence rates in tabular settings, but the log-linear policy class remains open (see \Cref{sec:policy regret appendix}).
Policy gradient is also frequently designed from an OMD perspective \citep{geist2019theory,shani2020adaptive}, where various entropy regularizers exist \citep{alfano2023novel,cichocki2025mirror}. However, we are unaware of existing works utilizing hypentropy in policy gradient to promote low-rankness of the learned parameter.

\paragraph{Parameter-Free Online Learning.} 
Parameter-free online learning attains near-optimal regret without prior information on the comparator \citep{chaudhuri2009parameter,streeter2012no,orabona2013dimension,mcmahan2014unconstrained}. \citet{orabona2016coin} proposed a ``coin-betting'' framework, which enables black-box frameworks for parameter-free online learning in richer regimes \citep{cutkosky2018black}. More generally, adaptivity to, e.g., optimal action's loss \citep{freund1997decision,allenberg2006hannan,allen2018make}, feedback delays \citep{zimmert2020optimal,gyorgy2021adapting,huang2023banker}, or properties of loss distributions \citep{huang2022adaptive,genalti2024varepsilon,chen2025uniinf}, has been widely studied in bandits or online learning.

\section{Setup: Contextual Bandits with Low-Rank Experts}\label{sec:setup}
\paragraph{Notations.} For an integer $n\in \mathbb Z_+$, $[n]$ denotes the set $\{1,2,\ldots,n\}$. We use $\O$ to hide absolute constants, and use $\Otil$ to additionally hide poly-logarithmic factors. Let $\mathbb S^n\subseteq \mathbb R^{n\times n}$ be the space of symmetric $n\times n$ real matrices. For $X\in \mathbb S^n$, $\lambda_i(X)$ denotes the $i$-th largest eigenvalue of $X$. Let $\lVert X\rVert_\ast:=\sum_i \lvert \lambda_i(X)\rvert$ be the nuclear norm and $\lVert X\rVert_2:=\max_i \lvert \lambda_i(X)\rvert$ be the spectral norm.

Consider a $T$-round game between a learner and the environment. In each round $t$, a \emph{query} embedded into the $d_q$-dimensional Euclidean space, namely $q_t\in \mQ:=\mathbb R^{d_q}$, is presented to the learner as a context; we allow $q_t$ to be arbitrary, i.e., chosen by an adaptive adversary. Although real queries can be in natural language, practical model routing pipelines first embed them via a ``base encoder'' \citep[see, e.g.,][]{lee2025routerretriever}.
The learner holds a large set of candidate \emph{items} $\mA\subseteq \mathbb R^{d_a}$, the action set in bandits, and aims to recommend an $a_t\in \mA$ that best answers $q_t$, measured by the expected \emph{reward} $r(q_t,a_t)$ that \Cref{sec:true reward} models as a similarity score on a latent space.
Instead of directly working on $\mA$, the learner is equipped with $M$ embedding \emph{models} as experts, each carrying a recommendation distribution $\xi_m(q_t)$ induced by its latent internal structure (\Cref{sec:expert reward}). Given the cost of invoking large models, however, the learner can query only \emph{one} model $m_t$ for a single item $a_{t,m_t}$ and receives bandit feedback on the induced reward alone \citep{tsiourvas2025causal}; \Cref{sec:feedback} details the protocol.

\subsection{Misspecified Adversarial Contextual Linear Bandit}\label{sec:true reward}

For query $q\in \mQ$ and item $a\in \mA$, we assume the expected reward $r(q,a)$ is roughly their {similarity score} $\langle U^\ast q,V^\ast a\rangle$ on a high-dimensional latent space $\mathbb R^{\max(d_a,d_q)}$, where $U^\ast\in \mathbb R^{\max(d_a,d_q)\times d_q}$ and $V^\ast\in \mathbb R^{\max(d_a,d_q)\times d_a}$ are fixed but unknown linear maps. Writing $\langle A,B\rangle_F:=\tr(A^\trans B)$ for the Frobenius inner product of matrices, this is equivalently a $d_a\times d_q$-dimensional linear structure:
\begin{equation}\label{eq:true reward}
r(q,a)=\bigl \langle \Psi^\ast,aq^\trans\bigr \rangle_F+\nu(q,a),\quad \text{where kernel }\Psi^\ast:=(V^\ast)^\trans U^\ast\in \mathbb R^{d_a\times d_q},
\end{equation}
and $\nu$ is a misspecification term capturing the non-linearity, with $\lVert \nu\rVert_\infty:=\sup_{q,a} \lvert \nu(q,a)\rvert\le \epsilon$ for some fixed but unknown $\epsilon\in [0,1)$ -- a standard way of modeling misspecification in linear bandits \citep{gopalan2016low,ghosh2017misspecified,foster2020adapting} or linear Markov decision processes \citep{du2020good,vial2022improved,jin2023provably}. All of $U^\ast$, $V^\ast$, $\nu$, and $\epsilon$ are unknown to the learner.

\subsection{Embedding Models as Low-Rank Experts}\label{sec:expert reward}
As standard both in practical recommendation systems and theoretical research, we model embedding models via the two-tower architecture \citep{huang2013learning}: Each model $m\in [M]$ projects queries and items independently onto its own embedding space $\mathbb R^{s_m}$, and recommends items to queries according to the softmax of similarity scores on $\mathbb R^{s_m}$. As observed by \citet{weller2025theoretical}, embedding models work in very low-dimensional representation spaces, hence we assume $s_m\ll d_a,d_q$.
When presented with query $q\in \mQ$, model $m$'s \emph{recommendation distribution} over the item set $\mA\subseteq \mathbb R^{d_a}$ is:%
\footnote{Although embedding models are internally non-linear \citep{hendrycks2016gaussian,devlin2019bert}, the linear form in \Cref{eq:expert reward initial} is backed by the Platonic representation hypothesis and the representation alignment literature \citep{lenc2015understanding,kornblith2019similarity,huh2024position}, which find the latent spaces of heterogeneous networks to be linearly inter-mappable. Specifically, \citet{wang2020understanding} prove that contrastive learning enforces an \emph{inner-product} geometry, in which case the residual non-linearity is nicely absorbed by $\mu_m$.}
\begin{equation}\label{eq:expert reward initial}
\xi_{m}(a\mid q)\propto \exp\Bigl (\bigl \langle U_mq,V_ma\bigr \rangle+\mu_m(q,a)\Bigr ),\quad \forall a\in \mA,
\end{equation}
which parallels the reward in \Cref{eq:true reward}: $U_m\in \mathbb R^{s_m\times d_q}$ and $V_m\in \mathbb R^{s_m\times d_a}$ are two fixed but unknown linear maps, and $\mu_m(q,a)$ is a misspecification function capturing the non-linearity in the embedding model. We assume $\lVert \mu_m\rVert_\infty:=\sup_{q,a}\lvert \mu_m(q,a)\rvert\le \epsilon$ for the same constant $\epsilon\in [0,1)$.
The key difference is that $U_m$ and $V_m$ are low-rank. For notational simplicity, we define $s:=\max_m s_m$ and omit the subscript $m$ in $s_m$. Consequently, when defining the \emph{model kernel} $\Psi_m:=V_m^\trans U_m$ similar to \Cref{eq:true reward}, we have $\rank(\Psi_m)\le s$ and
\begin{equation}\label{eq:expert reward}
\xi_m(a\mid q)\propto \exp\Bigl (\bigl \langle \Psi_m,aq^\trans\bigr \rangle_F+\mu_m(q,a)\Bigr ),\quad \text{where kernel }\Psi_m:=V_m^\trans U_m\text{ has a rank}\le s.
\end{equation}
All components in \Cref{eq:expert reward}, namely $\Psi_m,\mu_m$, and $s$, are unknown to the learner.
The name low-rank ``experts'' comes from the bandits with expert advice literature \citep{auer2002nonstochastic}.

\paragraph{Regularity.}
We assume that all the embeddings are $\ell_2$-normalized as $\lVert q\rVert_2,\lVert a\rVert_2\le 1$, $\forall q\in \mQ,a\in \mA$; all model kernels ensure $\lVert \Psi_m\rVert_2\le \rho$ for some $\rho\in [0,1]$; and the reward kernel satisfies $\lVert \Psi^\ast\rVert_2\le 1$.\footnote{The first condition is a standard practice in contrastive learning \citep[\S3]{wang2020understanding}. The second condition arises because, prior to computing softmax over inner products, practitioners usually divide the logits by a large temperature, as common in Transformers \citep{vaswani2017attention}, knowledge distillation \citep{hinton2015distilling}, and model calibration \citep{guo2017calibration}. Even if a sharper retrieval softmax pushed $\rho$ towards $1$, it only enlarges the residual term in \Cref{thm:quadratic policy class}; the main takeaway -- quadratic leading term -- is unaffected. The third condition ensures bounded expected rewards and is standard in the literature \citep{jin2023provably,luo2021policy}.}

\subsection{Interaction Protocol, Feedback Model, and Policy Regret}\label{sec:feedback}

Before the game, the environment fixes the reward kernel $\Psi^\ast$, model kernels $\{\Psi_m\}_m$, and misspecification functions $\nu$ and $\mu_m$ in \Cref{eq:true reward,eq:expert reward}.
The learner only knows the query space $\mQ=\mathbb R^{d_q}$, item set $\mA\subseteq \mathbb R^{d_a}$, number of models $M$, and number of rounds $T$.\footnote{While practical hard constraints may make the item set time-varying \citep{covington2016deep,wang2018billion,huang2020embedding}, incorporating such non-stationarity via sleeping bandits \citep{kleinberg2010regret} is orthogonal to our primary focus. We assume a stationary item set $\mA$ to isolate our core challenge of online embedding model routing. The knowledge of $T$ is for the ease of presentation, which can be relaxed via standard doubling tricks \citep{besson2018doubling}. The low-rank parameter $s=\max_m s_m$ in \Cref{eq:expert reward} is unknown.}
In each round $t=1,2,\ldots,T$:
\begin{enumerate}
\item the environment arbitrarily (perhaps adaptively adversarially) decides a query $q_t\in \mQ$;
\item each model $m\in [M]$, observing the query $q_t$, crafts its distribution $\xi_m(q_t)$ via \Cref{eq:expert reward};
\item the learner, observing query $q_t$ but not the distributions, chooses a model $m_t\in [M]$ to invoke;
\item the item recommended for this round is sampled from model $m_t$'s distribution as $a_{t,m_t}\sim \xi_{m_t}(q_t)$;
\item the learner observes a noisy reward $r_t=r(q_t,a_{t,m_t})+\eta_t$, where $r(q,a)$ is defined in \Cref{eq:true reward}. We assume $\eta_t$ is a zero-mean and conditionally 1-sub-Gaussian noise \citep{abbasi2011improved}.
\end{enumerate}

The feedback is doubly limited. Beyond the \emph{bandit reward feedback} common to recommendation systems \citep{li2010contextual,chapelle2011empirical} and to bandits with expert advice \citep{auer2002nonstochastic,beygelzimer2011contextual}, model routing suffers an \emph{expert observability} constraint: the learner sees neither what an unselected model $m\neq m_t$ would have recommended, nor the full distribution $\xi_{m_t}(q_t)$ behind the one item it does see. Assuming such counterfactuals is unrealistic given the cost of invoking multiple models \citep{tsiourvas2025causal,zu2026barouter}. \Cref{sec:expert regret,sec:contextual regret} show how the two together constitute significant challenges.

We consider the policy regret \citep{agarwal2014taming,syrgkanis2016efficient}, which compares our learner to a fixed class of model routing policies. Formally, a \emph{policy} is a mapping from query set $\mQ$ to the probability simplex over all models, i.e., $\triangle:=\{p\in [0,1]^M \mid \sum_{m=1}^M p_m=1\}$. The space of all policies is $\triangle^\mQ$. The policy regret is defined w.r.t. a fixed subset of policies, namely $\Pi\subseteq \triangle^{\mathcal Q}$, as:%
\footnote{Despite an adaptive adversary, \Cref{eq:policy regret} fixes the realized queries $q_1,q_2,\ldots,q_T$ and seeks the hindsight optimal $\pi^\ast\in \Pi$. The \emph{strict} policy regret evaluating $\pi^\ast$ on its induced queries must be $\Omega(T)$ \citep{arora2012online}.}
\begin{equation}\label{eq:policy regret}
\mR_T(\Pi):=\sup_{\pi^\ast\in \Pi} \E\left [\sum_{t=1}^T \E_{m\sim \pi^\ast(q_t)}\left [\E_{a\sim \xi_m(q_t)}[r(q_t,a)]\right ]-\sum_{t=1}^T r(q_t,a_t)\right ].
\end{equation}
Properly deciding $\Pi$ recovers several commonly used regret metrics: When $\Pi=\triangle^{\mathcal Q}$ takes the full policy space, it recovers the contextual regret, which is unfortunately intractable; when $\Pi$ restricts all queries to a fixed model $m^\ast\in [M]$, it reduces to the expert regret, which essentially ignores the heterogeneity of queries; when $\Pi=\Pi_{\text{lin}}$ is the standard log-linear policy class, it turns out to be structurally misaligned with the underlying problem structure.
In \Cref{sec:regret}, we study the first question of this paper: what is the most appropriate choice of $\Pi$ for embedding model routing?

\section{Policy Class: Tradeoff between Alignment and Tractability}\label{sec:regret}
Instead of more common policy classes (discussed later in this section), as a main contribution of this paper, we argue that the most appropriate choice of $\Pi$ is the following \emph{log-quadratic} policy class:
\begin{equation}\label{eq:quadratic policy class}
\Pi_{\text{quad}}:=\bigl \{\pi(m\mid q)\propto \exp(q^\trans W_m q),\forall q\in \mQ\mathrel{\big \vert} W_m\in \mathbb S^{d_q},\forall m\in [M]\bigr \},
\end{equation}
where $\mathbb S^{d_q}$ is the space of symmetric $d_q\times d_q$ matrices.
$\Pi_{\text{quad}}$ defines a parametric class of model routing policies: The probability of routing a query $q$ to a model $m$ is determined by a softmax distribution over the \emph{quadratic score} $q^\trans W_mq$ (hence the name).
We first justify that this log-quadratic class is well-aligned with the underlying structure of the embedding model routing problem.

\paragraph{Log-Quadratic Class.}\label{sec:quadratic policy class}
In \Cref{thm:quadratic policy class}, we establish that in embedding model routing, the expected reward of any model $m$ under any query $q$ is governed by a \emph{low-rank quadratic} structure:
\begin{proposition}[True Rewards are Roughly Low-Rank Quadratic]\label{thm:quadratic policy class}
Given a query $q\in \mQ=\mathbb R^{d_q}$, the expected reward of any model $m\in [M]$ is approximated by a low-rank quadratic form in $q$, namely
\begin{equation*}
R_m(q):=\E_{a\sim \xi_m(q)}[r(q,a)]=C(q)+q^\trans W_m^\ast q+\delta_m(q),\quad \forall m\in [M],q\in \mQ\subseteq \mathbb R^{d_q},
\end{equation*}
where $C(q)$ is a constant shared across all models (i.e., independent to $m$), $W_m^\ast\in \mathbb S^{d_q}$ is a symmetric matrix with a nuclear norm $\lVert W_m^\ast\rVert_\ast\le \rho\rank(W_m^\ast)\le 2s\rho$, and $\delta_m(q)=\O(\epsilon+\rho^2)$ is a residual term due to the non-linearity of $\nu$ and $\mu_m$ and to the softmax curvature $\lVert \Psi_m\rVert_2\le \rho$.
\end{proposition}
The proof is in \Cref{sec:quadratic policy class appendix}.
In light of \Cref{thm:quadratic policy class}, let us consider the following family of score-based embedding model router: Upon receiving a query $q$, the router calculates a score $s_m(q)$ for each model $m$, and performs routing via a softmax over scores with some temperature $c^{-1}$. Such an architecture is ubiquitous in classification or model routing systems \citep{devlin2019bert,fedus2022switch}.
The optimal score-based routing policy is setting the score as the true reward, i.e.,
\begin{equation}\label{eq:softmax over quadratic}
\pi^\ast(m\mid q)=\frac{\exp\bigl (cR_m(q)\bigr )}{\sum_{m'=1}^M \exp\bigl (c R_{m'}(q)\bigr )}\approx \frac{\exp\bigl (c q^\trans W_m^\ast q\bigr )}{\sum_{m'=1}^M \exp\bigl (c q^\trans W_{m'}^\ast q\bigr )}\propto \exp \Bigl (q^\trans \bigl(cW_m^\ast\bigr) q\Bigr ),~ \forall m.
\end{equation}

\Cref{thm:quadratic policy class} and \Cref{eq:softmax over quadratic} thus suggest that optimal model routing policies are naturally approximated by our log-quadratic class in \Cref{eq:quadratic policy class}, namely $\Pi_{\text{quad}}=\{\pi(m\mid q)\propto \exp(q^\trans W_mq)\mid W_m\in \mathbb S^{d_q},\forall m\}$, up to an residual $\delta_m(q)$.
In addition to well-aligning with the underlying structure of embedding model routing, $\Pi_{\text{quad}}$ is also tractable for learning: In \Cref{sec:OMD}, we design an efficient policy gradient algorithm, HPG, that attains $\Otil(s\sqrt{MT})$ \emph{linearized} policy regret (a surrogate metric we study; see \Cref{sec:parameter space}) w.r.t. $\Pi_{\text{quad}}$. We thus conclude that the log-quadratic $\Pi_{\text{quad}}$ \emph{offers an ideal trade-off for embedding model routing}: well-approximating the optimal policies, while enabling efficient learning. 

Having investigated our log-quadratic class, we examine why standard policy classes fall short. As depicted in \Cref{table} (whose log-quadratic baseline comes from flattening $\mW$ and invoking \citealt{agarwal2021theory}), we consider the following policy classes; detailed discussions are in \Cref{sec:reductions}.

\begin{table}[!t]
\centering
\caption{Comparison between Different Policy Classes or Regret Notions}\label{table}
\begin{minipage}{\textwidth}
\resizebox{\textwidth}{!}{%
\begin{savenotes}
\renewcommand{\arraystretch}{1.3}
\begin{tabular}{|c|c|c|c|c|c|}\hline
\textbf{Regret Notion} & \textbf{Aligned?} & \textbf{Tractable?} & \textbf{Regret} & \textbf{Computation} & \textbf{Comments} \\\hline
Expert Regret & \multirow{2}{*}{\xmark} & \multirow{2}{*}{\cmark} & $\sqrt{MT\log M}$ & $\O(M)$ & Non-Adaptive \\[-3pt]
{\color{black!50}(constant policy)} & & & {\small\citep{auer2002nonstochastic}} & {\small per round} & Model Routing\\
\hline
Contextual Regret & \multirow{2}{*}{\cmark} & \multirow{2}{*}{\xmark} & $\sqrt{d_ad_q T \log M \log T}$ & $\O(d_a^2d_q^2 M)$ & Curse of \\[-3pt]
{\color{black!50}(unrestricted policy}) & & & {\small \citep{li2019nearly}\footnote{This reduction only holds without misspecification ($\nu=\mu_m=0$) and when knowing model kernels ($\Psi_m$). ``Curse of dimensionality'' means, when treating $d_q$ and $d_a$ as the same order $d$, a regret bound that is linear in $d$.\label{footnote:curse of dimensionality}}} & {\small per round} & Dimensionality \\\hline
Log-Linear Regret & \multirow{2}{*}{\xmark} & \multirow{2}{*}{\cmark} & $\sqrt{d_qMT\log T}$ & $\O(d_q^2 M^2)$ & Structural \\[-3pt]
{\color{black!50}($\pi\propto \exp(\theta_m^\trans q)$)} & & & {\small\citep{agarwal2021theory}\footnote{Their bound, as well as ours, only applies to \emph{linearized} policy regret (\Cref{sec:parameter space}); an exact guarantee under either class remains open.}} & {\small per round} & Misspecification \\\hline
\multirow{4}{*}{\shortstack{Log-Quadratic\\Policy Regret\\{\color{black!50}($\pi\propto \exp(q^\trans W_mq)$)}}} & \multirow{4}{*}{\cmark} & \multirow{4}{*}{\cmark} & $d_q\sqrt{MT\log T}$ & $\O(d_q^4M^2)$ & Curse of \\[-3pt]
& & & {\small\citep{agarwal2021theory}} & {\small per round} & Dimensionality \\\cline{4-6}
& & & {\color{green!30!black}$s\sqrt{MT\log T}$} & {\color{green!30!black}$\O(d_q^2 M)$} & \textbf{\color{green!40!black}Efficient \&} \\[-3pt]
& & & \textbf{\color{green!40!black}(Ours)} & {\small\color{green!30!black} per round} & \textbf{\color{green!40!black}Parameter-Free}\\\hline
\end{tabular}
\end{savenotes}
}\renewcommand{\footnoterule}{}
\end{minipage}
\end{table}

\paragraph{Constant Class is Non-Adaptive (\Cref{sec:expert regret appendix}).}\label{sec:expert regret}
When $\Pi=\Pi_{\text{const}}$ only contains constant policies $\pi(q)\equiv m$, we recover the \emph{expert regret} of bandits with expert advice, whose benchmark answers every query with the \emph{same} model. It is not merely coarse but structurally uninformative: \Cref{thm:expert regret lower bound} establishes $\mR_T(\Pi_{\text{const}})=\Omega(\sqrt{MT/\log d_a})$ even when the learner observes all models' recommendations, so a black-box EXP3 \citep{auer2002nonstochastic} -- ignoring linear rewards and low-rank kernels entirely -- is already near-optimal.

\paragraph{Unrestricted Class gives ``Curse of Dimensionality''  (\Cref{sec:contextual regret appendix}).}\label{sec:contextual regret}
The other extreme, the unrestricted $\Pi=\triangle^\mQ$, recovers the contextual regret and is favorable in that the learner competes with \emph{any} query-adaptive policy -- yet it stays notorious even under \textbf{(A1)} no misspecification and \textbf{(A2)} known model kernels. Under both, $\mR_T(\triangle^\mQ)$ reduces to a $d_ad_q$-dimensional contextual linear bandit, on which VCL-SupLinUCB \citep{li2019nearly} attains $\mR_T^C=\O(\sqrt{d_ad_qT \log M \log T})$ regret at $\O(d_a^2d_q^2 M)$ per-round cost. Neither of them is acceptable once query and item spaces are large: with $d_q=d_a=768$ as in \textsc{RouterRetriever} \citep{lee2025routerretriever,izacard2022unsupervised}, the bound stays vacuous for $T/(\log M\log T)\le 768^2\approx 5\times 10^5$ rounds and each round costs $3\times 10^{11}M$ operations -- the \emph{``curse of dimensionality''} of Footnote~\ref{footnote:curse of dimensionality}.
Other seemingly more promising reductions (e.g., to $d_a$-dimensional adversarial contextual linear bandits, or to $Md_q^2$-dimensional ones via the lifted feature $q_tq_t^\trans$) also fail to refine the regret: even assuming i.i.d. queries on top of \textbf{(A1)} and \textbf{(A2)}, the correlation between recommendation distributions $\xi_m(\cdot \mid q_t)$ and rewards $r(q_t,\cdot)$ makes existing algorithms inapplicable, and the learner never chooses $q_t$ (see \Cref{sec:contextual regret adv,sec:contextual regret lifted}).

\paragraph{Log-Linear Class is Misaligned (\Cref{sec:policy regret appendix}).}\label{sec:policy regret}
In reinforcement learning, the commonly used class is the \emph{log-linear} $\Pi_{\text{lin}}:=\{\pi(m\mid q)\propto \exp(\theta_m^\trans q)\mid \theta_m\in \mathbb R^{d_q},\forall m\}$ \citep{agarwal2021theory,mei2020global}, whose $Md_q$-dimensional parameter space is far smaller than that of $\Pi_{\text{quad}}$. Based on Natural Policy Gradient (NPG; \citealt{kakade2001natural}), \citet{agarwal2021theory} attain $\O(\sqrt{d_q M T \log T})$ \emph{linearized} policy regret at $\O(d_q^2 M^2)$ per round,\footnote{Their linearization also resolves the non-convexity of softmax, but happens in a different space from ours; see \Cref{sec:parameter space} and \Cref{sec:policy regret appendix,sec:linearized policy regret appendix} for more detailed discussions.} shrinking the vacuous regime above from $768^2$ to $768M$ rounds. $\Pi_{\text{lin}}$ is in fact already deployed: \textsc{RouterRetriever} \citep{lee2025routerretriever} routes by the cosine similarity between $q_t$ and each model's ``representative training data,'' essentially a $\theta_m^\trans q$ score.
Yet $\Pi_{\text{lin}}$ is structurally misaligned: \Cref{thm:quadratic policy class} makes the true rewards \emph{quadratic} in $q$, and a linear score cannot capture interactions between coordinates of $q$. Even an optimal policy regret w.r.t. $\Pi_{\text{lin}}$ therefore leaves the learner sub-optimal at the actual routing task.

\subsection{Numerical Verification on Amazon ESCI Dataset}
In \Cref{tab:esci_direct_label_gaps}, we compare the optimal routing policies within four policy classes -- the constant $\Pi_{\text{const}}$, the log-linear $\Pi_{\text{lin}}$, the log-quadratic $\Pi_{\text{quad}}$, and the unrestricted $\triangle^\mQ$ -- over the Amazon ESCI dataset \citep{reddy2022shopping}.
This dataset benchmarks 97,345 difficult search queries in e-commerce: Each consumer query is associated with several candidate products (18.68 on average) with titles, descriptions, and keywords, all in natural language.
Each query-action pair is annotated with one of E(xact), S(ubstitute), C(omplement), or I(rrelevant), which we convert into numerical values as the reward function $r\colon \mQ\times \mA\to [0,1]$.
We consider $M=8$ lightweight embedding models as candidate experts, spanning symmetric semantic encoders like \texttt{miniLM} \citep{wang2020minilm,wang2021minilmv2}, asymmetric search models like \texttt{msmarco-distilbert} \citep{reimers2019sentence}, and modern prompt-driven encoders \citep{wang2022text,xiao2024c}. Detailed setup and discussions are in \Cref{sec:esci_direct_label_gaps}.

To facilitate model routing, we embed queries and items into 768-dimensional vector spaces via \textsc{Contriever} \citep{izacard2022unsupervised} before feeding them to the learner; this follows the practice of \textsc{RouterRetriever} \citep{lee2025routerretriever}.
We measure a routing policy $\pi$ by its \emph{sub-optimality gap} $\text{Gap}(\pi)$: averaged uniformly over ESCI queries, the expected reward that $\pi$ forfeits relative to routing each query to its own best model (see \Cref{eq:ESCI gap definition} for the formal definition).

\begin{table}[t]
\centering
\caption{Sub-Optimality Gap of Different Policy Classes on Amazon ESCI Dataset}
\label{tab:esci_direct_label_gaps}
\resizebox{\textwidth}{!}{%
\begin{tabular}{|c|c|c|c|c|c|c|}\hline
{Query} & {Num of} & $\Pi_{\text{const}}$ & $\Pi_{\text{lin}}$ & $\Pi_{\text{const}}$ $\to$ $\Pi_{\text{lin}}$ & $\Pi_{\text{quad}}$ & $\Pi_{\text{lin}}$ $\to$ $\Pi_{\text{quad}}$ \\
{Category} & {Queries} & {Gap ($\downarrow$)} & {Gap ($\downarrow$)} & {Improvement ($\uparrow$)} & {Gap ($\downarrow$)} & {Improvement ($\uparrow$)} \\\hline
all & 97,345 & 0.035 & 0.030 & 13.8\% & \textbf{0.011} & \textbf{63.5\%} \\\cdashline{1-7}[1pt/3pt]
\texttt{nlqec} & 169 & 0.110 & 0.004 & 96.0\% & \textbf{<0.001} & \textbf{99.7\%} \\
\texttt{behavioral} & 1,253 & 0.049 & 0.006 & 87.8\% & \textbf{<0.001} & \textbf{97.5\%} \\
\texttt{parse-pattern} & 4,250 & 0.046 & 0.014 & 69.8\% & \textbf{0.002} & \textbf{89.0\%} \\
\texttt{negations} & 2,899 & 0.134 & 0.054 & 59.9\% & \textbf{0.011} & \textbf{80.1\%} \\
other & 88,774 & 0.031 & 0.026 & 17.4\% & \textbf{0.009} & \textbf{66.2\%} \\\hline
\end{tabular}}
\end{table}

\Cref{tab:esci_direct_label_gaps} reports $\text{Gap}(\pi^\ast_{\text{const}})$, $\text{Gap}(\pi^\ast_{\text{lin}})$, and $\text{Gap}(\pi^\ast_{\text{quad}})$ for the optimal policy of each class, together with the relative improvements $1-\text{Gap}(\pi_{\text{lin}}^\ast)/\text{Gap}(\pi_{\text{const}}^\ast)$ and $1-\text{Gap}(\pi_{\text{quad}}^\ast)/\text{Gap}(\pi_{\text{lin}}^\ast)$; how each optimum is computed is deferred to \Cref{sec:esci_direct_label_gaps}. The first gap measures whether routing is needed at all: a large $\text{Gap}(\pi_{\text{const}}^\ast)$ means no single model dominates the others on most queries. Besides the full dataset, we report a few ``categories'' that \citet{reddy2022shopping} single out as exceptionally hard.

The pattern is consistent across rows: the log-linear class recovers only a small part of the constant-policy gap, whereas the log-quadratic class closes most of what remains. The contrast is sharpest where routing matters most: on the harder categories such as \texttt{nlqec} and \texttt{negations}, where the constant policy is far worse than its full-dataset average, $\Pi_{\text{quad}}$ still closes over $90\%$ of that gap. This evidences both the necessity of dynamic model routing and the low misspecification of $\Pi_{\text{quad}}$.

\subsection{Additional Notations and Linearized Policy Regret}\label{sec:parameter space}
We therefore focus on our log-quadratic $\Pi_{\text{quad}}$ class, optimizing over its parameters rather than over $\triangle$ as bandit convex optimization would, since a point chosen in $\triangle$ cannot depend on the query.
Since each policy is parameterized by $M$ symmetric $d_q\times d_q$ matrices, we collect them into one diagonal block matrix $\bm W=\diag(W_1,W_2,\ldots,W_M)\in \mathbb S^{Md_q}$, which we call a \emph{parameter} and denote by a boldface letter; the parameter space $\mW\subseteq \mathbb S^{Md_q}$ thus has dimension $Md_q^2$.
According to \Cref{eq:quadratic policy class}, each parameter $\bm W\in \mW$ induces a model routing policy:
\begin{equation}\label{eq:quadratic policy}
\pi_{\bm W}(m\mid q):= \frac{\exp(q^\trans W_{m} q)}{\sum_{m'=1}^M \exp(q^\trans W_{m'}q)},\quad \forall m=1,2,\ldots,M.
\end{equation}
For each round $t\in [T]$, define the expected loss (negative reward) suffered by any parameter $\bm W$ as
\begin{equation}\label{eq:loss function}
\mathcal L_t(\bm W):=-\E_{m\sim \pi_{\bm W}(q_t)}\left [\E_{a\sim \xi_m(q_t)}\left [r(q_t,a)\right ]\right ],\quad \forall t\in [T],\bm W\in \mW.
\end{equation}
Then the policy regret w.r.t. $\Pi_{\text{quad}}$, namely $\mR_T(\Pi_{\text{quad}})$, is $\sup_{\bm W^\ast\in \mW}\E[\sum_{t=1}^T \mathcal L_t(\bm W_t)-\mathcal L_t(\bm W^\ast)]$.
From \Cref{thm:quadratic policy class}, we assume that an optimal $\bm W^\ast\in \mW$ ensures $\lVert W_m^\ast\rVert_\ast\le 2s$, $\forall m$.

Unfortunately, since softmax is non-convex, neither policy (\Cref{eq:quadratic policy}) nor loss (\Cref{eq:loss function}) is convex. Consistent with previous policy gradient analysis (detailed in \Cref{sec:policy regret appendix}; \citealt{agarwal2021theory}), we replace $\mR_T(\Pi_{\text{quad}})$ by a surrogate obtained from \emph{local linearization} on the parameter space $\mW$:
\begin{equation}\label{eq:linearized policy regret}
\tilde \mR_T(\Pi_{\text{quad}}):=\sup_{\bm W^\ast \in \mW}\E\left [\sum_{t=1}^T \Bigl \langle \nabla \mathcal L_t(\bm W_t),\bm W_t-\bm W^\ast \Bigr \rangle_F\right ].
\end{equation}
Each summand recovers the Frank-Wolfe gap of online non-convex optimization \citep{lafond2015online,reddi2016stochastic}, common in, e.g., online submodular optimization \citep{hassani2017gradient,chen2018online,zhang2019online}.
See \Cref{sec:linearized policy regret appendix} for a more technical comparison between $\mR_T(\Pi_{\text{quad}})$ and our linearized $\tilde \mR_T(\Pi_{\text{quad}})$, whose one-step forms differ by an indefinite-sign factor $\log \frac{\pi^\ast}{\pi}$ that defeats gradient dominance arguments.
While we consider $\tilde \mR_T$ as a surrogate metric to handle non-convexity, tackling $\mR_T$ exactly remains an important open question.

\section{Hypentropy Policy Gradient (HPG) for Model Routing}\label{sec:OMD}

Our algorithm for the log-quadratic policy class $\Pi_{\text{quad}}$ draws on the routing application only through the low-rankness that \Cref{thm:quadratic policy class} grants $\bm W^\ast$.
Moving from the log-linear class $\Pi_{\text{lin}}$ to $\Pi_{\text{quad}}$ improves approximation accuracy (see \Cref{thm:quadratic policy class} and \Cref{eq:softmax over quadratic}), but comes at a cost:
The parameter space dimension increases from  $\dim(\Theta)=d_qM$ to $\dim(\mW)=d_q^2M$. 
Hence, naively flattening $\mW$ into a $d_q^2M$-dimensional vector space and applying policy gradient methods gives $\O(d_q \sqrt{MT \log T})$ regret and requires $\O(d_q^4 M)$ computation per round \citep{agarwal2021theory}. Similar to the unrestricted class discussed in \Cref{sec:regret}, both bounds can be prohibitively large.

Fortunately, since embedding models operate on low-dimensional representation spaces, \Cref{thm:quadratic policy class} makes the optimal parameter $\bm W^\ast$ low-rank, which is the key to bypassing the curse of dimensionality: restricting learning to low-rank $W_m$'s reduces the effective dimension to $\Theta(sd_qM)$. Exploiting this is not routine, though. The low-rank structure lies in the \emph{parameter} space, not the \emph{reward} space -- the kernel $\Psi^\ast$ in \Cref{eq:true reward} is \emph{not} low-rank, so low-rank or sparse regression (ridge, lasso) does not apply -- and standard policy gradient is blind to it: NPG, equivalent to Online Mirror Descent (OMD) with a KL regularizer \citep{geist2019theory}, is isotropic across directions and does not adapt to the rank.\footnote{Vanilla NPG is only defined on the flattened vector space, where the low-rank structure of $\bm W^\ast$ is lost. Exploiting low-rankness requires working with matrix spaces, but the matrix analogue of NPG -- OMD with von Neumann entropy $\Phi(X)=\tr(X\log X-X)$ -- requires \emph{positive semi-definite} matrices \citep{tsuda2005matrix}.}

We therefore use the hypentropy regularizer \citep{ghai2020exponentiated} to promote low-rankness (\Cref{sec:hypentropy}), paired with a REINFORCE-style gradient estimator for our bandit feedback and partial observability (\Cref{sec:REINFORCE}); together they give \textbf{H}ypentropy \textbf{P}olicy \textbf{G}radient (HPG) in \Cref{alg:OMD}. The estimator is unbiased and low-variance, giving provable guarantees under incomplete information (\Cref{thm:main theorem}), and its structure also lets \Cref{sec:practical implementation} implement HPG in $\O(d_q^2M)$ time per round and parameter-freely.

\subsection{Nuclear-Norm Ball, Hypentropy, and Online Mirror Descent}\label{sec:hypentropy}

\begin{algorithm}[!t]
\caption{\textbf{H}ypentropy \textbf{P}olicy \textbf{G}radient (HPG) for Model Routing}\label{alg:OMD}
\begin{algorithmic}[1]
\Require{Game length $T$, query set $\mQ\subseteq \mathbb R^{d_q}$, candidate item set $\mA\subseteq \mathbb R^{d_a}$, number of models $M$.\\
Hypentropy parameter $\beta>0$, learning rate $\eta>0$, nuclear norm bound $\tau>0$.}
\State Initialize $\bm W_1$ as the $Md_q\times Md_q$ all-zero matrix, which is a diagonal block matrix in $\mathbb S^{Md_q}$.
\For{$t=1,2,\ldots,T$}
\State Observe query $q_t$. Decide model $m_t$ according to $\pi_{\bm W_t}(q_t)$ \Comment{$\pi_{\bm W}$ is in \Cref{eq:quadratic policy}}
\State Observe model $m_t$'s recommendation $a_{t,m_t}\sim \xi_{m_t}(q_t)$ \Comment{$\xi_m$ is in \Cref{eq:expert reward}}
\State Recommend item $a_{t,m_t}$. Observe $r_t=r(q_t,a_{t,m_t})+\eta_t$ \Comment{$r(q,a)$ is in \Cref{eq:true reward}}
\State Construct gradient estimator $\hat{\bm G}_t$ using $q_t,m_t,r_t$ \Comment{$\hat{\bm G}_t$ is in \Cref{eq:REINFORCE}}
\State Perform an Online Mirror Descent (OMD) step over $\mathbb B_\ast^M(\tau)$. \Comment{\Cref{eq:product nuclear norm ball,eq:hypentropy,eq:OMD}}\label{line:OMD}
\EndFor
\end{algorithmic}
\end{algorithm}

Since the optimal parameter $\bm W^\ast$ is low-rank, the first component of HPG restricts the parameter space to a \emph{product nuclear-norm ball}:
\begin{equation}\label{eq:product nuclear norm ball}
\mathbb B_\ast^M(\tau):=\bigl\{ \diag(W_1,W_2,\ldots,W_M)\mathrel{\big \vert} W_m\in \mathbb S^{d_q},\lVert W_m\rVert_\ast\le \tau,\forall m\bigr \}.
\end{equation}

Here $\lVert \cdot \rVert_\ast$, the nuclear norm (sum of absolute eigenvalues), is the usual convex surrogate for rank, and $\tau\ge 0$ is to be determined: since \Cref{thm:quadratic policy class} gives $\lVert W_m^\ast\rVert_\ast\le 2s\rho\le 2s$, we have $\bm W^\ast\in \mathbb B_\ast^M(2s)$. For ease of presentation we assume $s$ known for now, and drop this in \Cref{sec:parameter-free}.
HPG then runs OMD on $\mathbb B_\ast^M(\tau)$ with the hyperbolic entropy (\emph{hypentropy}; \citealt{ghai2020exponentiated}), which is strongly convex with respect to the nuclear norm, making it well-suited for promoting low-rank structures in the parameter space.\footnote{The connection between hypentropy and low-rankness has been exploited by \citet{woodworth2020kernel}, \citet{pesme2021implicit}, \citet{li2022implicit}, \citet{varre2023spectral}, and \citet{jacobs2025mirror} to understand the implicit bias/regularization in neural network training, by \citet{wu2020continuous,wu2023nearly} in sparse phase retrieval problems, and by \citet{wu2021implicit} for low-rank matrix sensing. However, we are unaware of existing works using this property for policy regret, bandit feedback, or embedding model routing.}
Formally, for a parameter $\beta>0$ to be specified later, the $\beta$-hypentropy of a symmetric matrix $\bm W\in \mathbb S^{Md_q}$ is defined as
\begin{equation}\label{eq:hypentropy}
\Phi_\beta(\bm W):=\sum_{i=1}^{Md_q} \left (\lambda_i \arcsinh \frac{\lambda_i}{\beta}-\sqrt{\lambda_i^2+\beta^2}\right ),\quad \forall \bm W\in \mathbb S^{Md_q},
\end{equation}
where $\lambda_1\ge \lambda_2\ge \cdots \ge \lambda_{Md_q}$ are the eigenvalues of $\bm W$ in non-increasing order.
Using $\Phi_\beta$ as the regularizer, the OMD update of HPG over the product ball $\mathbb B_\ast^M(\tau)$ takes the following form:
\begin{equation}\label{eq:OMD}
\bm W_{t+1}=\argmin_{\bm W\in \mathbb B_\ast^M(\tau)}\Bigl (\eta \bigl \langle \bm W,\hat{\bm G}_t\bigr \rangle_F +D_\Phi^\beta(\bm W\|\bm W_t)\Bigr ),\quad \forall t=1,2,\ldots,T,
\end{equation}
where $\eta>0$ is a parameter to be determined; $\hat{\bm G}_t$ is the (estimated) gradient defined in \Cref{sec:REINFORCE}; and $D_\Phi^\beta(\bm W\|\bm W'):=\Phi_\beta(\bm W)-\Phi_\beta(\bm W')-\langle \nabla \Phi_\beta(\bm W'),\bm W-\bm W'\rangle_F$ (detailed in \Cref{sec:main theorem appendix notations}).
This Bregman projection, usually computationally expensive, is implementable in $\O(d_q^2M)$ time thanks to the hypentropy regularizer $\Phi_\beta$ and the nuclear-norm ball $\mathbb B_\ast^M(\tau)$ (\Cref{sec:efficient implementation}).

\subsection{Gradient Estimation and Policy Regret Bound}\label{sec:REINFORCE}
The other component of HPG is the gradient estimator $\hat{\bm G}_t$. Full-information OMD would put the true $\nabla\mathcal L_t(\bm W_t)$ in its place in \Cref{eq:OMD}, but under bandit feedback and partial observability we see only one item $a_{t,m_t}\sim \xi_{m_t}(q_t)$ and one noisy reward $r_t$, and must estimate the gradient from that alone.
We design a REINFORCE-style gradient estimator: let $\hat{\bm G}_t=\diag(\hat G_{t,1},\hat G_{t,2},\ldots,\hat G_{t,M})$, where
\begin{equation}\label{eq:REINFORCE}
\hat G_{t,m}=-r_t \bigl (\mathbbm{1}[m_t=m]-\pi_{\bm W_t}(m\mid q_t)\bigr )q_tq_t^\trans,\quad \forall m=1,2,\ldots,M.
\end{equation}

\begin{lemma}[Gradient Estimation]\label{lem:REINFORCE}
Let $\E_t$ be the expectation taken only w.r.t. the randomness in round $t$, i.e., conditional on the natural filtration up to round $(t-1)$. Then we have $\E\nolimits_t[\hat{\bm G}_t\mid q_t]=\nabla_{\bm W} \mathcal L_t(\bm W_t)$. Furthermore, $\sum\nolimits_{m=1}^M \lVert \hat G_{t,m}\rVert_2^2\le 2r_t^2$ almost surely, and consequently $\E\nolimits_t[\sum\nolimits_{m=1}^M \lVert \hat G_{t,m}\rVert_2^2\mid q_t]\le 10$.
\end{lemma}
\Cref{lem:REINFORCE} indicates that $\hat{\bm G}_t$ is conditionally unbiased and admits constant conditional second-order moments.
While similar estimators and properties are widely used in policy gradient methods \citep[see, e.g.,][]{williams1992simple,sutton1999policy}, our log-quadratic class has a further property: the $\hat G_{t,m}$ in \Cref{eq:REINFORCE} is \emph{rank-one}, which is pivotal for the computational efficiency of HPG in \Cref{sec:efficient implementation}.

Combining this gradient estimator with OMD updates yields the full HPG algorithm (see \Cref{alg:OMD}).
\begin{theorem}[Regret Bound]\label{thm:main theorem}
Assume that the optimal $\bm W^\ast\in \mW$ ensures $\lVert W_m^\ast\rVert_\ast\le 2s$, $\forall m$. When setting $\tau=2s$, $\beta=2s/d_q$ and $\eta=\sqrt{(M\log d_q)/T}$, HPG gives $\tilde \mR_T(\Pi_{\text{quad}})=\O(s\sqrt{MT\log d_q})$.
\end{theorem}

\Cref{thm:main theorem} indicates that if the low-rank parameter $s$ is known to the learner, the HPG algorithm attains $\Otil(s\sqrt{MT})$ linearized policy regret against the best in-class parameter, thus provably adapting to the underlying low-rank structures of embedding models. In \Cref{sec:parameter-free}, we further show that when the low-rank parameter $s$ is unknown, the HPG algorithm is still implementable in a \emph{parameter-free} way that attains an almost identical regret bound.
The full proofs are in \Cref{sec:REINFORCE appendix,sec:main theorem appendix}.

\section{Practical Implementation of the HPG Algorithm}\label{sec:practical implementation}
While attaining $\Otil(s\sqrt{MT})$ regret, two issues arise before implementing HPG for realistic model routing: the computational cost of the Bregman projection step in \Cref{eq:OMD}, and the knowledge of $s$ in \Cref{thm:main theorem}. We resolve both issues in this section, thus demonstrating the practicality of HPG.

\subsection{Computationally Efficient Implementation}\label{sec:efficient implementation}
The Bregman projection of \Cref{eq:OMD} is a convex program over the $M^2d_q^2$-dimensional space $\mathbb S^{Md_q}$, on which a vanilla implementation of the interior point method would spend $\O(M^6 d_q^6)$ time per round \citep{boyd2004convex} -- some $2\times 10^{17}M$ floating point operations at $d_q=768$.
Fortunately, since the hypentropy $\Phi_\beta$ and nuclear-norm ball $\mathbb B_\ast^M(\tau)$ are both defined w.r.t. eigenvalues, and diagonal block matrices preserve eigenvalues, we establish a \emph{direct sum property} (\Cref{lem:direct sum} in \Cref{sec:direct sum appendix}). This allows us to treat each model's parameter separately.
Furthermore, we prove that each $m$ only requires $\O(d_q^2)$ time per round by maintaining an eigen-decomposition of $W_{t,m}$: 
\begin{equation}\label{eq:SVD of W_t}
W_{t,m}=U_{t,m}\diag(\lambda_{t,m})U_{t,m}^\trans,\quad \forall t\in [T],
\end{equation}
where $U_{t,m}\in \mathbb R^{d_q\times d_q}$ and $\lambda_{t,m}\in \mathbb R^{d_q}$. Moving between the primal and mirror spaces is free: as detailed in \Cref{sec:main theorem appendix notations}, $\nabla \Phi_\beta$ and its inverse simply apply $\arcsinh(\cdot/\beta)$ and $\beta \sinh(\cdot)$ to the eigenvalues in $\O(d_q)$ time. The remaining of \Cref{eq:OMD} then decomposes into two easy steps:
\begin{enumerate}
\item \textbf{Gradient Descent in Mirror Space.} Perform $\tilde Y_{t+1,m}=\nabla \Phi_\beta(W_{t,m})-\eta \hat G_{t,m}$ and map back to $\tilde W_{t+1,m}=(\nabla \Phi_\beta)^{-1}(\tilde Y_{t+1,m})$. Eigen-decomposing $\tilde Y_{t+1,m}$ generally takes $\O(d_q^3)$ time, but as $\hat G_{t,m}$ is rank-one (\Cref{eq:REINFORCE}), \emph{eigen-updating} from $W_{t,m}$ takes only $\O(d_q^2)$ \citep{bunch1978rank}.
\item \textbf{Projection onto $\mathbb B_\ast(\tau)$ based on $D_\Phi^\beta$.} Usually expensive as well, but our $\mathbb B_\ast(\tau)$ and $\Phi_\beta$ both depend only on the \emph{eigenvalues}, so von Neumann's trace inequality gives
\begin{equation}\label{eq:SVD of W_t+1}
W_{t+1,m}=U_{t+1,m}\diag(\lambda_{t+1,m})U_{t+1,m}^\trans,\quad \lambda_{t+1,m}=\argmin_{\lVert \lambda\rVert_1\le \tau} \sum_{i=1}^{d_q} D_\phi^\beta(\lambda_i\|\tilde \lambda_{t+1,m,i}),
\end{equation}
where $D_\phi^\beta$ is the Bregman divergence of $\phi_\beta(x)=x\arcsinh \frac x\beta-\sqrt{x^2+\beta^2}$: a $d_q$-dimensional convex program, solvable in $\O(d_q\log d_q)$ time.
\end{enumerate}

In \Cref{sec:efficient implementation appendix}, we prove \Cref{eq:OMD} is equivalent to these steps. This gives the following theorem:
\begin{theorem}[Computational Cost]\label{lem:efficient implementation}
HPG is implementable in $\O(d_q^2 M)$ time complexity per round.
\end{theorem}
In practice one can do better still: \Cref{sec:efficient implementation appendix} shows the projection step is itself low-rank, so eigen-truncating to the top-$\tau$ eigenvalues brings the eigen-update down to $\O(\tau d_q)$, and OMD updates can be batched in the mirror space at the cost of the strict guarantee.

\subsection{Parameter-Free Implementation}\label{sec:parameter-free}
The remaining issue is the nuclear-norm parameter $\tau$ in \Cref{alg:OMD}: In \Cref{thm:main theorem}, we set $\tau$ to $2s$ where $s$, the maximum rank of model kernels, is \emph{unknown} to the learner. Fortunately, utilizing the parameter-free online learning technique \citep{cutkosky2018black} and a clipping argument, the HPG algorithm can be implemented without \emph{any} prior information on $s$ while enjoying a similar regret guarantee. The detailed algorithm description and proof are deferred to \Cref{sec:parameter-free appendix}.
\begin{theorem}[Parameter-Free HPG]\label{thm:parameter-free}
Assume that the optimal $\bm W^\ast\in \mW$ ensures $\lVert W_m^\ast\rVert_\ast\le 2s$, $\forall m$, but the low-rank parameter $s$ is unknown. Then the parameter-free HPG algorithm (\Cref{alg:parameter-free}) ensures a linearized policy regret bound of $\tilde \mR_T(\Pi_{\text{quad}})=\O(s\sqrt{MT\log (d_q T)}+sM\sqrt{\log T}\log (sT))$.
\end{theorem}

Once $T\gtrsim M\log T$, the first term dominates and this matches the $\O(s\sqrt{MT\log d_q})$ in \Cref{thm:main theorem} up to poly-logarithmic factors. HPG's favorable properties all remain: adapting to the unknown low-rank structure, avoiding $\text{poly}(d_q)$ dependencies, and admitting an efficient implementation.

\section{Conclusion and Future Research}
In this paper, we formalize the online embedding model routing problem as an adversarial contextual linear bandit with low-rank experts. One key challenge is that the learner only invokes a single model per round, and observes neither the full recommendation distribution nor what an unselected model would have returned.
The standard benchmarks all fail here: The constant class is non-adaptive, the unrestricted class is intractable, the log-linear class is structurally misaligned. We instead identified a log-quadratic class by proving each model's expected reward is a low-rank quadratic form (\Cref{thm:quadratic policy class}); we numerically verified our claim on Amazon ESCI dataset.
For the resulting log-quadratic policy optimization problem, which can be of independent interest, Hypentropy Policy Gradient (HPG) attains $\Otil(s\sqrt{MT})$ linearized policy regret with no $\text{poly}(d_q)$ dependence, can be computationally efficiently implemented, and does not require knowing the low-rank parameter $s$.
Transferring our guarantees from the linearized surrogate $\tilde \mR_T$ to the exact policy regret $\mR_T$ is left for future investigation, where an indefinite-sign factor $\log(\pi^\ast/\pi)$ defeats gradient dominance arguments. Hardness results -- whether an $\Omega(s)$-dependency is necessary for $\Pi_{\text{quad}}$ -- also remain open.

\bibliography{references}
\bibliographystyle{plainnat}

\appendix
\onecolumn
\renewcommand{\appendixpagename}{\centering \LARGE Technical Appendices}
\appendixpage

\startcontents[section]
\printcontents[section]{l}{1}{\setcounter{tocdepth}{2}}

\section{Setup of \Cref{fig:gemini_v_embeddings} and More Discussions}\label{sec:experiment}

In \Cref{fig:gemini_v_embeddings}, we consider a recommendation system scenario over the Amazon Berkeley Objects (ABO) dataset \citep{collins2022abo}. Based on the ``keywords'' in the ABO dataset, we randomly picked 500 items with one of the following keywords: \emph{`sofa', `couch', `chair', `loveseat', `ottoman.'}
We consider the following specific user prompt, which is generated by a Large Language Model (LLM) API, Gemini-2.5-pro, to mimic a typical customer that has a diverse range of requirements: ``\emph{Looking for a modern, dark-colored lounge sofa (not a chair). Definitely NOT leather. Prefer fabric or linen with a cozy but structural design.}'' It contains various components: vague style descriptions (modern, cozy but structural), color (dark-colored), object category (lounge sofa, not a chair), and material constraints (NOT leather, prefer fabric or linen).

In industrial recommendation systems like Youtube \citep{covington2016deep}, Alibaba e-commerce \citep{wang2018billion}, and Facebook \citep{huang2020embedding}, a filter step taking care of hard constraints is implemented independent of the embedding models. We capture this by invoking Gemini-2.5-pro with the following prompt: ``\emph{You are a strict data filter for a recommendation system.
A user has provided the following search query: \{user\_prompt\}
Task:
1. Identify any STRICT negative constraints or hard requirements in the user's query (e.g., if they explicitly say "NOT X", then X is a hard constraint).
2. Evaluate the following items. If an item clearly violates a strict negative constraint from the query, it fails the filter.
3. Return ONLY a valid JSON dictionary where keys are Item IDs and values are boolean `true` (passes constraints, or no strict constraints violated) or `false` (fails strict constraints).}'' The items satisfying all hard constraints in the user prompt stated above -- 131 out of 500 -- are passed onto the next stage of recommendation.

For the recommendation stage, we use two lightweight embedding models pre-trained using different methodologies:
\textit{(i)} the \texttt{miniLM} model by \citet[\texttt{all-MiniLM-L6-v2} on Huggingface]{wang2020minilm,wang2021minilmv2}, whose embeddings are optimized so that semantically similar sentences have high cosine similarity; and
\textit{(ii)} the \texttt{msmarco} model by \citet[\texttt{msmarco-distilbert-base-v4} on Huggingface]{reimers2019sentence}, an asymmetric retrieval model trained on a large-scale dataset built from Bing search queries \citep{bajaj2016ms}.
Prior benchmarks show that text embedding models are strongly task-dependent: models optimized for semantic similarity and models optimized for query-document retrieval perform differently across downstream retrieval and similarity tasks \citep{thakur2021beir,muennighoff2023mteb}. 
Thus, using one symmetric semantic model and one asymmetric retrieval model is expected to produce different rankings over the same candidate items.

To evaluate these two models, we again called the LLM API, Gemini-2.5-pro, with the following prompt: ``\emph{You are a recommendation system evaluator. A user queried: \{user\_prompt\}
Note: All items below have already passed the hard filter constraints of the query.
Evaluate how perfectly each item satisfies the remaining soft constraints (modern, dark-colored, cozy but structural, fabric/linen).
Task: Score each item from 0 to 100.
Return ONLY a valid JSON dictionary where keys are Item IDs and values are integer scores.}'' The resulting ranking is treated as the ground-truth.

In \Cref{fig:gemini_v_embeddings}, we visualize the rankings using a 2D scatter plot, with the x-axis and y-axis representing the ranks assigned by \texttt{miniLM} and \texttt{msmarco}, respectively. To show the discrepancy between different models without visual clutter, low-relevance items (ground-truth rank $> 40$, which is roughly 25\%) are plotted as faint crosses, while high-quality items (ground-truth rank $\le 40$) are shown as gray dots. Furthermore, we also highlight representative ``exclusively good'' items identified as follows: An item ranked in the top tier by both the LLM ground-truth and one embedding model, while simultaneously being penalized by the other embedding model. Green dots are items uniquely captured by \texttt{miniLM}, and red dots are items uniquely captured by \texttt{msmarco}. This gives our main motivation: No single model dominates the other, and hence dynamic model routing is necessary.

\section{More Discussions on Policy Classes and Regret Notions}\label{sec:reductions}

\subsection{Log-Quadratic Policy Class (\Cref{thm:quadratic policy class})}\label{sec:quadratic policy class appendix}
\begin{proof}[Proof of \Cref{thm:quadratic policy class}]
We begin by defining the ideal (non-misspecified) recommendation distribution for expert $m$, which removes the misspecification function $\mu_m$ in \Cref{eq:expert reward}:
\begin{equation*}
\xi_m^\ast(a\mid q) = \frac{\exp\bigl (\langle \Psi_m, a q^\trans \rangle\bigr )}{\sum_{a'\in \mathcal A} \exp\bigl (\langle \Psi_m, a' q^\trans \rangle\bigr )},\quad \forall m\in [M],q\in \mQ,a\in \mA.
\end{equation*}
Given the bounded misspecification $\lVert \mu_m \rVert_\infty \le \epsilon$ in \Cref{eq:expert reward}, the ratio between the actual distribution and the ideal distribution is strictly bounded. Specifically, for any $a \in \mA$:
\begin{align*}
\frac{\xi_m(a\mid q)}{\xi_m^\ast(a\mid q)} &= \exp\bigl(\mu_m(q,a)\bigr) \frac{\sum_{a'\in \mA} \exp\bigl (\langle \Psi_m, a' q^\trans \rangle\bigr )}{\sum_{a'\in \mA} \exp\bigl (\langle \Psi_m, a' q^\trans \rangle + \mu_m(q,a')\bigr )} \\
&\le \exp(\epsilon) \frac{\sum_{a'\in \mA} \exp\bigl (\langle \Psi_m, a' q^\trans \rangle\bigr )}{\exp(-\epsilon) \sum_{a'\in \mA} \exp\bigl (\langle \Psi_m, a' q^\trans \rangle\bigr )} = \exp(2\epsilon).
\end{align*}
By symmetry, the lower bound is $\exp(-2\epsilon)$. Consequently, the total variation distance is bounded by
\begin{equation*}
D_{\text{TV}}(\xi_m, \xi_m^\ast) \le \frac{1}{2} \sum_a \xi_m^\ast(a\mid q) \lvert \exp(2\epsilon) - 1 \rvert = \O(\epsilon).
\end{equation*}

We therefore compare the expected action of model $m$ when answering query $q$, namely $\bar a_m(q) := \E_{a\sim \xi_m}[a]$, to that under the ideal $\xi_m^\ast$, namely $\bar a_m^\ast(q) := \E_{a\sim \xi_m^\ast}[a]$. We have
\begin{equation*}
\lVert \bar a_m(q) - \bar a_m^\ast(q) \rVert_2 = \Bigl\lVert \sum_{a\in \mA} a \bigl(\xi_m(a\mid q) - \xi_m^\ast(a\mid q)\bigr) \Bigr\rVert_2 \le \max_{a\in \mA} \lVert a\rVert_2 \cdot \sum_{a\in \mA} \lvert \xi_m - \xi_m^\ast \rvert \le \O(\epsilon),
\end{equation*}
where we invoked the regularity assumption that $\lVert a\rVert_2 \le 1$ (\Cref{sec:expert reward}).

Because $\xi_m^\ast$ belongs to the exponential family, the expected action is the gradient of the log-partition function evaluated at $\Psi_m q$, i.e.,
\begin{equation*}
\bar a_m^\ast(q)=\E_{a\sim \xi_m^\ast}[a]=\nabla A(\Psi_mq),\quad \text{where }A(z):=\log \sum_{a\in \mathcal A} \exp(\langle z,a\rangle).
\end{equation*}

Expand $\nabla A(z)$ around $z=0$ to the second order:
\begin{equation*}
\nabla A(z) = \nabla A(0) + \nabla^2 A(0) z + R_2(z),
\end{equation*}
where $\lVert R_2(z)\rVert_2 \le \O(\lVert z\rVert_2^2)$.
Due to the log-partition function, $\nabla A(0) = \frac{1}{\lvert \mA\rvert}\sum_{a\in \mA} a := \mu_{\mA}$ is the mean of all items, and $\nabla^2 A(0) = \frac{1}{\lvert \mA\rvert}\sum_{a\in \mA} (a-\mu_{\mA})(a-\mu_{\mA})^\trans := \Sigma_{\mA}$ is the covariance matrix. 

Now, plug in $z = \Psi_m q$. By the regularity conditions that $\lVert q \rVert_2 \le 1$ and $\lVert \Psi_m \rVert_2 \le \rho$, we have $\lVert z \rVert_2 \le \rho$. Thus the remainder is bounded by $\lVert R_2(z) \rVert_2 \le \O(\rho^2)$. Combining it with the total variational bound between $\bar a_m(q)$ and $\bar a_m^\ast(q)$, we have
\begin{equation*}
\bar a_m^\ast(q) = \mu_{\mA} + \Sigma_{\mA} \Psi_m q + \tilde e_m(q), \quad \bar a_m(q) = \mu_{\mA} + \Sigma_{\mA} \Psi_m q + e_m(q),
\end{equation*}
where $\lVert \tilde e_m(q) \rVert_2,\lVert e_m(q) \rVert_2 \le \O(\epsilon+\rho^2)$.
Substituting $\bar a_m(q)$ into the reward function in \Cref{eq:true reward},
\begin{align*}
\E_{a\sim \xi_m(q)}[r(q,a)] &= \E_{a\sim \xi_m(q)}\bigl[ q^\trans (\Psi^\ast)^\trans a + \nu(q,a) \bigr] = q^\trans (\Psi^\ast)^\trans \bar a_m(q) + \E_{a\sim \xi_m}[\nu(q,a)] \\
&= \underbrace{q^\trans (\Psi^\ast)^\trans \mu_{\mA}}_{\text{constant $C(q)$}} + \underbrace{q^\trans (\Psi^\ast)^\trans \Sigma_{\mA} \Psi_m q}_{\text{quadratic $q^\trans W_m q$}} + \underbrace{q^\trans (\Psi^\ast)^\trans e_m(q) + \E_{a\sim \xi_m}[\nu(q,a)]}_{\text{residual $\delta_m(q)$}}.
\end{align*}

We first focus on the quadratic term $q^\trans W_m q$ where $W_m:=(\Psi^\ast)^\trans \Sigma_{\mA} \Psi_m$. Since the quadratic form only depends on the symmetric part, we define $W_m^\ast=\frac 12(W_m+W_m^\trans)\in \mathbb S^{d_q}$. It is also low-rank:
\begin{equation*}
\rank(W_m^\ast) = \rank\left(\frac{1}{2}W_m + \frac{1}{2}W_m^\trans\right) \le \rank(W_m) + \rank(W_m^\trans) \le 2\rank(\Psi_m)\le 2s_m \le 2s.
\end{equation*}

Consequently, we have the following nuclear-norm bound:
\begin{equation*}
\lVert W_m^\ast\rVert_\ast\le \rank(W_m^\ast) \times \lVert W_m^\ast\rVert_2\le 2s\times \frac 12 \left (\lVert W_m\rVert_2+\lVert W_m^\trans\rVert_2 \right )\le 2s \times 1\times 1\times \rho=2s\rho,
\end{equation*}
where the last step uses the regularity conditions that $\lVert \Psi^\ast\rVert_2\le 1$, $\lVert a\rVert_2\le 1$ for all $a\in \mA$, and $\lVert \Psi_m\rVert_2\le \rho$.
For the residual term $\delta_m(q)$, because $\lVert q \rVert_2 \le 1$, $\lVert \Psi^\ast \rVert_2 \le 1$, and $\lvert \nu \rvert \le \epsilon$,
\begin{equation*}
\lvert \delta_m(q) \rvert \le \lVert q \rVert_2 \cdot \lVert \Psi^\ast \rVert_2 \cdot \lVert e_m(q) \rVert_2 + \sup_{a\in \mA} \nu(q,a) \le 1 \cdot 1 \cdot \O(\epsilon+\rho^2) + \epsilon = \O(\epsilon+\rho^2).
\end{equation*}

This finishes the proof: We have $\E_{a\sim \xi_m(q)}[r(q,a)]=C(q)+q^\trans W_m^\ast q+\delta_m(q)$, such that $W_m^\ast \in \mathbb S^{d_q}$ satisfies $\lVert W_m^\ast\rVert_\ast\le \rho\rank(W_m^\ast)\le 2s\rho$, and that $\lvert \delta_m(q)\rvert\le \O(\epsilon+\rho^2)$.
\end{proof}

\subsection{Constant Policy Class and Expert Regret}\label{sec:expert regret appendix}

When picking the reference policy class as the constant class, the policy regret recovers the expert regret in the problem of bandits with expert advice (BwE). We begin by defining the BwE problem.
\begin{definition}[Bandits with Expert Advice; \citealt{auer2002nonstochastic}]\label{def:expert regret}
Consider a $T$-round $K$-armed bandit between a learner and an environment. In each round $t\in [T]$, the learner observes $M$ experts' advice $\xi_{t,1},\xi_{t,2},\ldots,\xi_{t,M}$, which are probability distributions over $[K]$. The learner chooses one expert $m_t$ and plays according to their advice, i.e., $a_t\sim \xi_{t,m_t}$. The learner suffers an adversarial loss $\ell_{t,a_t}$.
The learner minimizes \emph{expert regret}, the sub-optimality gap compared to the optimal expert in hindsight:
\begin{equation}\label{eq:expert regret}
\mR_T^E:=\max_{m^\ast\in [M]} \E\left [\sum_{t=1}^T \left ( \ell_{t,a_t}-\E_{a\sim \xi_{t,m^\ast}}[\ell_{t,a}]\right )\right ].
\end{equation}

In the more restrictive \emph{limited observability} setup \citep{seldin2013open,kale2014multiarmed}, there is a parameter $N\le M$. The learner can only pick a size-$N$ subset of experts in each round to observe their advice. Everything else, including the performance metric of expert regret in \Cref{eq:expert regret}, remains the same.
\end{definition}

In our embedding model routing problem (\Cref{sec:setup}), if picking $\Pi$ as the class of constant policies:
\begin{equation}\label{eq:constant policy class}
\Pi_{\text{const}}:=\bigl\{\pi(m\mid q)=\mathbbm{1}[m=m_0],\forall q\in \mQ\mathrel{\big \vert}m_0\in [M]\bigr\},
\end{equation}
which is parameterized by a single parameter $m_0\in [M]$, our problem is almost \Cref{def:expert regret} with $N=1$: Let $K=\lvert \mA\rvert$, i.e., each item $a\in \mA$ is an arm. View each model as an expert, with the round-$t$ advice $\xi_{t,m}$ defined according to the $\xi_m(\cdot \mid q_t)$ in \Cref{eq:expert reward}. The loss of arm $a$ in round $t$ is the $-r(q_t,a)$ in \Cref{eq:true reward}. The reason why we say ``almost'' is the correlation between expert advice and losses: In the standard bandits-with-expert-advice model \citep[see, e.g.,][Footnote 2]{kale2014multiarmed}, the advice and losses must be independent conditional on the history (i.e., they can both be adaptive to the history, but cannot be correlated with each other). Fortunately, our claims in \Cref{sec:expert regret} still hold.

On the upper bound side, we ignore the ``experts'' structure and apply an EXP3 algorithm \citep{auer2002nonstochastic}: In round $t$, the learner chooses one between $M$ models, namely $m_t$, and the loss of model $m$ in round $t$ has a mean $\E_{a\sim \xi_m(q_t)}[r(q_t,a)]$, which is bounded by $[-1,1]$. Since EXP3 allows adaptive losses, its policy regret w.r.t. $\Pi_{\text{const}}$ is $\O(\sqrt{MT\log M})$ \citep[Corollary 3.2]{auer2002nonstochastic}.

On the lower bound side, we translate the construction of \citet[Theorem 4]{kale2014multiarmed} into our setting:
\begin{theorem}[Regret Lower Bound w.r.t. $\Pi_{\text{const}}$]\label{thm:expert regret lower bound}
There exists a group of instances with $d_a\gg M$, $d_q=Md_a$, and $s=d_a$, such that the policy regret w.r.t. $\Pi_{\text{const}}$ must be $\Omega(\sqrt{MT/\log d_a})$.
Moreover, even if we allow the learner to query multiple $N\le M$ models per round, the same bound remains.
\end{theorem}

The proof of \Cref{thm:expert regret lower bound} is presented at the end of this section.
We hence conclude that the constant policy class and the induced expert regret are inappropriate for embedding model routing: A black-box EXP3 ignoring the linear reward structures in \Cref{eq:true reward} and the low-rank model structures in \Cref{eq:expert reward} attains near-optimal regret. This lower bound holds even with full observability on models (as long as $M\le d_a$, which is typically true since $d_a=768$ in \textsc{RouterRetriever}; \citealt{lee2025routerretriever}).

We give two more remarks before concluding this section. First, one may ask with full observability, whether the Hedge algorithm \citep{freund1997decision} -- attaining $\O(\sqrt{T\log M})$ regret -- becomes applicable. The answer is no: while full observability on models is obtained, the reward feedback is still bandit (see \Cref{sec:feedback}).
Second, noticing that the $s$ in \Cref{thm:expert regret lower bound} is chosen to be the same as $d_a$, one may ask whether the low-rank bilinear or matrix bandit algorithms \citep{jun2019bilinear,jang2021improved,lu2021low,kang2022efficient} can be used to derive sharpened regret bounds and bypass the $\Omega(\sqrt{MT/\log d_a})$ lower bound. The answer is again no: the reward kernel $\Psi^\ast$ in \Cref{eq:true reward} can be general (especially, having a rank as large as $d$), and the learner also has no control over the query $q_t$.

\subsubsection{Lower Bound on Policy Regret w.r.t. Constant Policy Class}
\begin{proof}[Proof of \Cref{thm:expert regret lower bound}]
Write $d_q=MK$ and $d_a=K$. Let $\mA=\{e_1,e_2,\ldots,e_K\}\in \mathbb R^{K}$ be unit vectors in the item space. For each round $t\in [T]$, the adversary samples $k_{t,1},k_{t,2},\ldots,k_{t,M}$ uniformly at random from $[K]$. The query is given as $q_t=\frac{1}{\sqrt M}[e_{k_{t,1}}^\trans,e_{k_{t,2}}^\trans,\ldots,e_{k_{t,M}}^\trans]^\trans\in \mathbb R^{MK}$.
Each expert $m\in [M]$ has the following kernel ($M$ matrices in total, and $\bm I_{K\times K}$ appears as the $m$-th):
\begin{equation*}
\Psi_m=\alpha \sqrt  M[\bm 0_{K\times K}\quad \cdots \quad \bm 0_{K\times K}\quad \bm I_{K\times K}\quad \bm 0_{K\times K}\quad \cdots \quad \bm 0_{K\times K}]\in \mathbb R^{K\times MK},
\end{equation*}
where $\alpha$ is a parameter controlling the softmax temperature in \Cref{eq:expert reward}. The rank of $\Psi_m$ is $K$. Setting misspecification $\mu_m\equiv 0$, the model $m$ then with very high probability recommends the item $k_{t,m}$:
\begin{equation*}
\xi_m(e_k\mid q_t)=\frac{\exp(\langle \Psi_m,e_k q_t^\trans\rangle_F)}{\sum_{k'=1}^K \exp(\langle \Psi_m,e_{k'} q_t^\trans\rangle_F)}=\frac{\exp(\alpha \mathbbm 1[k=k_{t,m}])}{(K-1)+\exp(\alpha)},\quad \forall k\in [K].
\end{equation*}

It only remains to construct the true rewards. Following the construction by \citet[Theorem 4]{kale2014multiarmed}, the environment picks a ``good'' expert $m^\ast\in [M]$ uniformly at random before the game, and set
\begin{equation*}
\Psi^\ast=\delta \sqrt M [\bm 0_{K\times K}\quad \cdots \quad \bm 0_{K\times K}\quad \bm I_{K\times K}\quad \bm 0_{K\times K}\quad \cdots \quad \bm 0_{K\times K}]\in \mathbb R^{K\times MK},
\end{equation*}
where there are $M$ matrices in total and $\bm I_{K\times K}$ appears as the $m^\ast$-th, and the $\delta>0$ is a parameter enabling information theoretic arguments. This $\Psi^\ast$ ensures $r(q_t,e_k)=\delta \mathbbm{1}[k=k_{t,m^\ast}]$. That is, only the item recommended by the $m^\ast$-th expert has a reward of $\delta$, whereas all others have zero reward.

We therefore resemble the hard instance constructed by \citet[Theorem 4]{kale2014multiarmed} by setting our $\alpha$ and $\delta$ according to their $\epsilon$. \citet[Theorem 4]{kale2014multiarmed} proved that in bandits with expert advice with limited observability, one suffers an expert regret lower bound of (with notations from \Cref{def:expert regret})
\begin{equation*}
\mR_T^E=\Omega \left (\sqrt{\frac{\min(K,N/\log K)}{N} MT}\right ).
\end{equation*}
Plugging in our configuration that $N\le M\ll d_a=K$, we come at the conclusion that, even if we allow full observability on models, the policy regret w.r.t. $\Pi_{\text{const}}$ (\Cref{eq:constant policy class}) is still at least
\begin{equation*}
\mR_T(\Pi_{\text{const}})=\Omega \left (\sqrt{\frac{N/\log d_a}{N}MT}\right )=\Omega \left (\sqrt{\frac{MT}{\log d_a}}\right ).\qedhere
\end{equation*}
\end{proof}

\subsection{Unrestricted Policy Class and Contextual Regret}\label{sec:contextual regret appendix}
Similar to \Cref{sec:expert regret appendix}, we begin by defining the problem of contextual linear bandits and contextual regret. We use the adversarial formulation by \citet{liu2023bypassing} to ease subsequent discussions.
\begin{definition}[Adversarial Contextual Linear Bandit; \citealt{liu2023bypassing}]\label{def:contextual linear bandit}
Consider a $T$-round game between a learner and an environment. In round $t$, a context, or equivalently, an action set $A_t\subseteq \mathbb R^d$, is presented to the learner. The learner picks action $a_t\in A_t$. At the same time, the environment (without observing $a_t$) chooses a loss vector $\ell_t\in \mathbb R^d$. The learner incurs and observes a noisy sample of the loss $\langle \ell_t,a_t\rangle$. The learner minimizes the contextual regret, defined as
\begin{equation}\label{eq:contextual regret}
\mR_T^C:=\sup_{\pi^\ast} \E\left [\sum_{t=1}^T \left (\langle \ell_t,a_t\rangle-\E_{a\sim \pi^\ast(A_t)}[\langle \ell_t,a\rangle]\right )\right ],
\end{equation}
where $\pi^\ast$ is any policy that maps a subset $A\subseteq \mathbb R^d$ to a (possibly randomized) action $a\in A$ therein.
\end{definition}

To illustrate the hardness of matching the unrestricted policy class $\triangle^\mQ$ via reduction to \Cref{def:contextual linear bandit}, throughout this section, we make the following two assumptions as mentioned in \Cref{sec:regret}:
\begin{enumerate}
\item[A1.] There is no misspecification in rewards or recommendation distributions so that everything is perfectly linear, i.e., $\nu(q,a)=\mu_m(q,a)=0$ for all $q\in \mQ$ and $a\in \mA$ in \Cref{eq:true reward,eq:expert reward}.
\item[A2.] The low-rank recommendation kernels of all models, namely the $\Psi_m$'s in \Cref{eq:expert reward}, are also known to the learner in advance. This allows the learner to calculate the following \emph{expected recommendation}:
\begin{equation*}
\bar a_m(q_t):=\E_{a\sim \xi_m(q_t)}[a]\in \mathbb R^{d_a},\quad \forall m\in [M],
\end{equation*}
which means this is an even stronger assumption than full observability on models (\Cref{thm:expert regret lower bound}).
\end{enumerate}
Since these two assumptions strictly ease learning, but matching the unrestricted policy class remains notoriously hard as we detail in subsequent sections, this metric is only harder in the original setup.

\subsubsection{Unrestricted Policy Class $\to $ A-S Contextual Linear Bandit}\label{sec:contextual regret stoc}

The first reduction, as sketched in \Cref{sec:contextual regret} in the main text, is to the adversarial-context stochastic-loss (``A-S'') contextual linear bandit problem. Specifically, take the $d$ in \Cref{def:contextual linear bandit} as $d=d_ad_q$. The round-$t$ action set (or context) is given as
\begin{equation*}
A_t=\bigl \{\text{vec}(\bar a_m(q_t)~q_t^\trans) \mathrel{\big \vert} m\in [M]\bigr \}\subseteq \mathbb R^{d}=\mathbb R^{d_ad_q},
\end{equation*}
where the $\text{vec}(\cdot)$ means flattening a $\mathbb R^{m\times n}$ matrix into a $\mathbb R^{mn}$ vector: $A_{i,j}=\text{vec}(A)_{(i-1)\times n+j}$.
Note that, assumption (A2) is used to ensure that the learner can calculate the action set $A_t$ in each round $t$.

When choosing action $a=\text{vec}(\bar a_m(q_t)~q_t^\trans)\in A_t$ in round $t$, which is equivalent to choosing model $m$ in round $t$, the expected loss (i.e., negative-reward) is
\begin{equation}\label{eq:linear bandit}
-\E_{a\sim \xi_{m}(q_t)}[r(q_t,a)]=-\Bigl \langle \Psi^\ast, \E_{a\sim \xi_{m}(q_t)}[a]~q_t^\trans\Bigr \rangle_F=-\Bigl\langle \text{vec}(\Psi^\ast),\text{vec}(\bar a_m(q_t)~q_t^\trans)\Bigr \rangle,
\end{equation}
where the first step uses the definition of $r(q,a)$ in \Cref{eq:true reward} and also assumption (A1) to ensure the linearity, and the second step uses the property that for two matrices $A,B\in \mathbb R^{m\times n}$,
\begin{equation*}
\langle A,B\rangle_F:=\tr(A^\trans B)=\sum_{i=1}^m\sum_{j=1}^n A_{i,j} B_{i,j}=\langle \text{vec}(A),\text{vec}(B)\rangle.
\end{equation*}

Therefore, the loss vector in \Cref{def:contextual linear bandit} is always chosen as $\ell_t=-\text{vec}(\Psi^\ast)$. This consequently reduces to an adversarial-context (because $A_t$'s depend on $q_t$ and are thus adversarial) and stochastic-loss (because $\ell_t$ is stationary) contextual linear bandit problem. We list a few available algorithms:
\begin{itemize}
\item OFUL \citep{abbasi2011improved}: $\O(d\sqrt{T}\log T)$ contextual regret with $\O(\lvert A_t\rvert d^2)$ computation per round. A faster ``rarely switching'' implementation is available, but the regret is the same.
\item SupLinUCB \citep{chu2011contextual}: $\O\left (\sqrt{dT\log^3(\lvert A\rvert T)}\right )$ contextual regret (where $\lvert A\rvert:=\max_t \lvert A_t\rvert$; in our case we thus have $\lvert A\rvert=M$) with $\O(\lvert A\rvert d^2\log T)$ computation per round.
\item VCL-SupLinUCB \citep{li2019nearly}: $\O(\sqrt{dT \log T \log \lvert A\rvert})$ contextual regret with $\O(\lvert A\rvert d^2\log T)$ computation per round.
\end{itemize}

As we argued in \Cref{sec:regret}, when $d=d_ad_q$ and $d_a=d_q=768$ \citep{lee2025routerretriever,izacard2022unsupervised}, none of these regret bounds or computational cost guarantees are acceptable.

\subsubsection{Unrestricted Policy Class $\to $ A-A Contextual Linear Bandit}\label{sec:contextual regret adv adv}

A seemingly more promising reduction is rewriting \Cref{eq:linear bandit} as a $d_a$-dimensional inner product:
\begin{equation*}
-\E_{a\sim \xi_{m}(q_t)}[r(q_t,a)]=-\Bigl \langle \Psi^\ast, \E_{a\sim \xi_{m}(q_t)}[a]~q_t^\trans\Bigr \rangle_F=-\langle \Psi^\ast q_t,\bar a_m(q_t)\rangle,
\end{equation*}
which corresponds to \Cref{def:contextual linear bandit} with round-$t$ action set $A_t=\bigl \{\bar a_m(q_t)\mathrel{\big \vert} m\in [M]\bigr \}\subseteq \mathbb R^{d_a}$ and loss vector $\ell_t=-\Psi^\ast q_t\in \mathbb R^{d_a}$.
While reducing the problem dimension greatly, the adversarial $q_t$ -- appearing in both action set $\mA_t$ and loss vector $\ell_t$ -- makes this problem intractable: Adversarial-context adversarial-loss (so-called ``A-A'') contextual linear bandits are computationally infeasible \citep{kanade2014learning}, admit $\Omega(\lvert \triangle^\mQ\rvert)$ oracle-complexity lower bounds \citep{hazan2016computational}, and are conjectured to bear $\Omega(T)$ regret lower bounds \citep{neu2020efficient}.

\subsubsection{Unrestricted Policy Class with i.i.d. Query $\to $ S-A Contextual Linear Bandit}\label{sec:contextual regret adv}
Given the intractability comes from the fact that both the action set $A_t$ and the loss vector $\ell_t$ are adversarial, what if we assume queries are i.i.d. in addition to the assumptions (A1) and (A2)?
\begin{enumerate}
\item[A3.] Assume that all queries $q_t$ are i.i.d. samples from a fixed but unknown distribution $\mathbb Q\in \triangle(\mQ)$.
\end{enumerate}

This assumption is reasonable in, e.g., large-scale e-commerce systems, where it is acceptable to expect that customers arrive uniformly at random from a large population.
Now the action set $A_t=\bigl \{\bar a_m(q_t)\mathrel{\big \vert} m\in [M]\bigr \}$ is i.i.d., and only the loss vector $\ell_t=-\Psi^\ast q_t$ is time-varying, does it fall into the tractable category of stochastic-context adversarial-loss (``S-A'') contextual linear bandits? (While $\ell_t$ is also stochastic, a stochastic-loss setup requires a stationary $\ell_t$.) This is because S-A contextual linear bandits admit an efficient algorithm with $\Otil(d^2\sqrt T)$ contextual regret \citep{liu2023bypassing}. In our case where action sets have bounded sizes, $\O(\sqrt{dT \log \lvert A\rvert})$ is also possible \citep{ito2024minimax}.

However, the answer is negative.
This is because all ``S-A'' contextual linear bandit algorithms require the ``ghost sample'' technique proposed by \citet{neu2020efficient}. Roughly speaking, via the independence between losses and contexts, this technique allows one to analyze the contextual regret by ``fixing'' a single $q_0\in \mathbb Q$ throughout the game. However, our $q_t$ appears both in $A_t=\{\bar a_m(q_t)\}_m$ and in $\ell_t=-\Psi^\ast q_t$, thus inducing correlation. All existing algorithms are hence inapplicable \citep{neu2020efficient,olkhovskaya2023first,liu2023bypassing,ito2024minimax,van2025improved}.

\subsubsection{Unrestricted Policy Class $\to $ Lifted-Feature Linear Bandit}\label{sec:contextual regret lifted}

The reductions above all take $\bar a_m(q_t)$ as the feature, which is where assumption (A2) is needed. Lifting the \emph{query} instead avoids it: \Cref{thm:quadratic policy class} gives
\begin{equation*}
\E_{a\sim \xi_m(q_t)}[r(q_t,a)]=C(q_t)+\bigl \langle W_m^\ast,q_tq_t^\trans\bigr \rangle_F+\delta_m(q_t),\quad \forall m\in [M],
\end{equation*}
which corresponds to \Cref{def:contextual linear bandit} with $d=Md_q^2$, unknown $\diag(W_1^\ast,W_2^\ast,\ldots,W_M^\ast)$, and $M$ actions placing $q_tq_t^\trans$ in the $m$-th block (the model-independent $C(q_t)$ cancels across actions). SupLinUCB \citep{chu2011contextual} or VCL-SupLinUCB \citep{li2019nearly} thus attain $\O(\sqrt{Md_q^2T\log T\log M})=\Otil(d_q\sqrt{MT})$ contextual regret, up to the $\O((\epsilon+\rho^2) T)$ contributed by the residuals $\delta_m$.

This is worse than the $\O(\sqrt{d_ad_qT\log M\log T})$ of \Cref{sec:contextual regret stoc} whenever $d_a$ and $d_q$ are of the same order, so the curse of dimensionality persists. Nor is it removed by the low-rankness of $W_m^\ast$: as in \Cref{sec:expert regret appendix}, the rank-dependent guarantees of low-rank bilinear and matrix bandits \citep{jun2019bilinear,jang2021improved,lu2021low,kang2022efficient} require \emph{learner-chosen} matrix actions, whereas $q_tq_t^\trans$ is decided by the adversary. Rank-dependent contextual regret hence remains out of reach, which motivates the log-quadratic policy class and HPG.

To conclude, matching the unrestricted policy class $\triangle^\mQ$ or minimizing the contextual regret $\mR_T^C$ is desirable for embedding model routing. However, it is intractable -- even after making the unrealistic assumptions of knowing model kernels and no misspecification -- due to the extreme flexibility of $\pi^\ast\in \triangle^\mQ$, which can impose an \emph{arbitrary} decision boundary when routing $\mQ$ to $\triangle$.

\subsection{Log-Linear Policy Class and Linearization in Policy Regret}\label{sec:policy regret appendix}

The log-linear policy class is formally defined as follows:
\begin{equation}\label{eq:linear policy class}
\Pi_{\text{lin}}:=\bigl \{\pi(m\mid q)\propto \exp(\theta_m^\trans q),\forall q\in \mQ\mathrel{\big \vert} \theta_m\in \mathbb R^{d_q},\forall m\in [M]\bigr \}.
\end{equation}

It assumes the routing policy is parameterized by $M$ $d_q$-dimensional vectors, $\bm \theta=(\theta_1,\theta_2,\ldots,\theta_M)\in (\mathbb R^{d_q})^M$. We call $\Theta:=(\mathbb R^{d_q})^M$ the parameter space. The policy induced by parameter $\bm \theta\in \Theta$ is thus
\begin{equation*}
\pi_{\bm \theta}(m\mid q)=\frac{\exp(\theta_m^\trans q)}{\sum_{m'=1}^M \exp(\theta_{m'}^\trans q)},\quad \forall m\in [M],q\in \mQ.
\end{equation*}

Equivalently, we can view $\bm \theta$ as a concatenated $Md_q$-dimensional vector and let
\begin{equation*}
\phi(q,m)=[0^\trans\quad \cdots \quad 0^\trans\quad q^\trans \quad 0^\trans\quad \cdots 0^\trans]
\end{equation*}
be the concatenation of $M$ $d_q$-dimensional vectors with the $m$-th one being $q$ and all remaining ones being $0$. Thus $\theta_m^\trans q=\bm \theta^\trans \phi(q,m)$, recovering the log-linear class by \citet[\S6.1.1]{agarwal2021theory}.

As discussed by \citet{agarwal2021theory}, $\Pi_{\text{lin}}$ is much harder than the \emph{tabular} softmax policy class $\Pi_{\text{tab}}$, where each policy $\pi(m\mid q)\propto \exp(\theta'_{q,m})$ for some $\bm \theta'\in \Theta':=\mathbb R^{M\lvert \mQ\rvert}$. Indeed, for $\Pi_{\text{tab}}$, it is possible to exactly handle the policy regret \citep[Theorems 5.1 and 5.4]{agarwal2021theory}. However, for the log-linear policy class $\Pi_{\text{lin}}$, where the policy $\pi(m\mid q)\propto \exp(\bm \theta^\trans \phi(q,m))$ for some $\bm \theta\in \Theta:=\mathbb R^{Md_q}$, \citet{agarwal2021theory} instead considered a \emph{linearized surrogate regret notion} instead of the policy regret. We now introduce their notations and results in more details. Since they focus on stochastic MDPs, we also assume i.i.d. queries (that is, $q_t\sim \mathbb Q$ with some fixed but unknown $\mathbb Q$) in this section.

For any parameter $\bm \theta\in \Theta$, let $\bar \phi_{\bm \theta}(q,m)=\nabla_{\bm \theta}\log \pi_{\bm \theta}(m\mid q)$. Now consider a round $t\in [T]$ with parameter $\bm \theta_t\in \Theta$. For each model $m\in [M]$, define the advantage under $\bm \theta_t$ as
\begin{equation*}
A_{\bm \theta_t}(q,m) = \E_{a\sim \xi_m(q)}[r(q,a)] - \sum_{m'=1}^M \pi_{\bm \theta_t}(m'\mid q)\E_{a\sim \xi_{m'}(q)}[r(q,a)].
\end{equation*}

The natural policy gradient (NPG) step in round $t$ is then given as \citep[\S6.1.1]{agarwal2021theory}
\begin{equation}\label{eq:NPG gradient}
\bm w_{t}:=\argmin_{\bm w\in \Theta} \E_{q\sim \mathbb Q, m\sim \pi_{\bm \theta_t}(\cdot\mid q)}\left[ \left( A_{\bm \theta_t}(q,m) - \bm w^\trans \bar \phi_{\bm \theta_t}(q,m) \right)^2 \right].
\end{equation}
\citet[Assumption 6.1 and Theorem 6.1]{agarwal2021theory} defined the following \emph{transfer error} $\epsilon_{\text{bias}}$, and had a final bound scaling with $\sqrt{\epsilon_{\text{bias}}}\times T$:\footnote{The bound of \citet{agarwal2021theory} is stated w.r.t. single-round sub-optimality gaps, hence a linear dependency on $T$ arises when converting to our cumulative regret setup. As an additional remark, their $\gamma$ is the discount factor that only applies to multi-stage MDPs; in our bandit setup, we have $\gamma=0$.}
\begin{align*}
\epsilon_{\text{bias}} &:= \max_{t\in [T]} \E_{q\sim \mathbb Q, m\sim \pi_{\bm \theta^\ast}(\cdot\mid q)}\left[ \left( A_{\bm \theta_t}(q,m) - \bm w_t^\trans \bar \phi_{\bm \theta_t}(q,m) \right)^2 \right].
\end{align*}

Let the expected reward -- when taking $q\sim \mathbb Q$ into consideration -- of policy $\pi_{\bm \theta_t}$ be $J(\bm \theta_t)$, defined as
\begin{equation*}
J(\bm \theta_t):=\E_{q\sim \mathbb Q}\left [\E_{m\sim \pi_{\bm \theta_t}(q)}\left [\E_{a\sim \xi_m(q)}[r(q,a)]\right ]\right ].
\end{equation*}

By direct arithmetic calculation (which is the famous Performance Difference Lemma specialized to bandits; \citealt{kakade2002approximately}), the one-step true policy regret in $\mR_T(\Pi_{\text{lin}})$ decomposes as
\begin{align*}
&\quad J(\bm \theta^\ast) - J(\bm \theta_t) = \E_{q\sim \mathbb Q}\left[ \sum_{m=1}^M \pi_{\bm \theta^\ast}(m\mid q) A_{\bm \theta_t}(q,m) \right]\\
&= \E_{q\sim \mathbb Q}\left[ \sum_{m=1}^M \pi_{\bm \theta^\ast}(m\mid q) \bm w_t^\trans \bar \phi_{\bm \theta_t}(q,m) \right]+ \E_{q\sim \mathbb Q}\left[ \sum_{m=1}^M \pi_{\bm \theta^\ast}(m\mid q) \left( A_{\bm \theta_t}(q,m) - \bm w_t^\trans \bar \phi_{\bm \theta_t}(q,m) \right) \right].
\end{align*}

Instead of tackling $J(\bm \theta^\ast)-J(\bm \theta_t)$ directly, \citet{agarwal2021theory} controlled the second term via Cauchy-Schwartz as $\sqrt{\epsilon_{\text{bias}}}$ and instead studied the first term.
For log-linear policy class, we have $\bar \phi_{\bm \theta_t}(q,m) = \phi(q,m) - \E_{m'\sim \pi_{\bm \theta_t}(\cdot\mid q)}[\phi(q,m')]$ \citep[\S6.1.1]{agarwal2021theory}. Thus, we know
\begin{equation*}
\bm w_t^\trans \bar \phi_{\bm \theta_t}(q,m) = w_{t,m}^\trans q - \sum_{m'=1}^M \pi_{\bm \theta_t}(m'\mid q) w_{t,m'}^\trans q.
\end{equation*}
Plugging this back into the first term and using the fact that $\sum_{m=1}^M \bm \theta^\ast(m\mid q) = 1$, we obtain:
\begin{align*}
\E_{q\sim \mathbb Q}\left[ \sum_{m=1}^M \pi_{\bm \theta^\ast}(m\mid q) \bm w_t^\trans \bar \phi_{\bm \theta_t}(q,m) \right] &= \E_{q\sim \mathbb Q}\left[ \sum_{m=1}^M \pi_{\bm \theta^\ast}(m\mid q) w_{t,m}^\trans q - \sum_{m'=1}^M \pi_{\bm \theta_t}(m'\mid q) w_{t,m'}^\trans q \right]\\
&=\E_{q\sim \mathbb Q}\left [\sum_{m=1}^M \left (\pi_{\bm \theta^\ast}(m\mid q)-\pi_{\bm \theta_t}(m\mid q)\right )w_{t,m}^\trans q\right ],
\end{align*}

\citet[Theorem 6.1 and Corollary 6.1]{agarwal2021theory} then proved that NPG attains $\sqrt T$-style regret under this \emph{linearized surrogate regret}, but suffering a linear $\sqrt{\epsilon_{\text{bias}}}\times T$ overhead in the true policy regret. Their linearization, however, happens in a different space from our $\tilde \mR_T(\Pi)$ in \Cref{eq:linearized policy regret}: theirs in the action space (the probability simplex), ours in the parameter space (the $\bm W\in \mW$ of \Cref{sec:parameter space}). It remains open whether these two linearized surrogate regret notions for the policy regret can be transformed to each other.

\subsection{Log-Quadratic Policy Class and Linearized Policy Regret}\label{sec:linearized policy regret appendix}
\begin{proposition}[Linearized Policy Regret vs Policy Regret]\label{exp:linearized policy regret}
Consider the log-quadratic policy class $\Pi_{\text{quad}}$ parameterized by $\bm W\in \mW$; see \Cref{sec:parameter space}.
For each round $t\in [T]$ and model $m\in [M]$, let $\pi_{t,m}:=\pi_{\bm W_t}(m\mid q_t)$, $\pi_{t,m}^\ast:=\pi_{\bm W^\ast}(m\mid q_t)$, and $r_{t,m}:=\E_{a\sim \xi_m(q_t)}[r(q_t,a)]$. Then the one-step regret in the policy regret w.r.t. $\Pi_{\text{quad}}$ is (where $\mathcal L_t$ is defined in \Cref{eq:loss function})
\begin{equation*}
\mathcal L_t(\bm W_t)-\mathcal L_t(\bm W^\ast)=\sum_{m=1}^M \pi_{t,m}^\ast \left (r_{t,m}-\sum_{m'=1}^M \pi_{t,m'} r_{t,m'}\right ),
\end{equation*}
while that in the linearized policy regret $\tilde \mR_T(\Pi_{\text{quad}})$ (defined in \Cref{eq:linearized policy regret}) is
\begin{equation*}
\Bigl \langle \nabla_{\bm W}\mathcal L_t(\bm W_t),\bm W_t-\bm W^\ast\Bigr\rangle_F=\sum_{m=1}^M \pi_{t,m} \left (r_{t,m}-\sum_{m'=1}^M \pi_{t,m'} r_{t,m'}\right ) \log \frac{\pi_{t,m}^\ast}{\pi_{t,m}}.
\end{equation*}
\end{proposition}
Given the very similar forms in \Cref{exp:linearized policy regret}, it is tempting to claim that these two regret notions are multiplicatively close to each other via a \emph{gradient dominance condition}; this unfortunately fails.
The gradient dominance condition is given as follows.\footnote{This type of condition, also known as \L{}ojasiewicz condition, has been developed by \citet{agarwal2021theory} and \citet{mei2020global} to exactly control the policy regret of NPG w.r.t. the tabular softmax class $\Pi_{\text{tab}}:=\{\pi(m\mid q)\propto \exp(\theta'_{q,m})\mid \bm \theta'\in \mathbb R^{M\lvert \mQ\rvert}\}$; see more in \Cref{sec:policy regret appendix}. We also remark their analysis only works with i.i.d. queries.} Due to the projection onto the product nuclear-norm ball $\mathbb B_\ast^M(\tau)$, the maximum and minimum possible coefficient of a model $m$ under any $W_m\in \mathbb B_\ast(\tau)$ are similar, hence giving $\max_m \frac{\pi_m^\ast}{\pi_m}\le \exp(2\tau)$. But the other term, $\log \frac{\pi_m^\ast}{\pi_m}$, is problematic: both it and the $(r_m-\sum_{m'=1}^M \pi_{m'} r_{m'})$ term (also known as the \emph{advantage function} in reinforcement learning) can be negative, hence the true policy regret is \emph{not} automatically bounded by $\exp(2\tau)$ times the linearized policy regret. We leave the exact policy regret for future research.
\begin{proof}[Proof of \Cref{exp:linearized policy regret}]
For notational simplicity, we focus on a single round $t\in [T]$ and omit all the subscript $t$'s.
By definition of $\mathcal L_t$ in \Cref{eq:loss function}, the one-step regret in $\mR_T(\Pi_{\text{quad}})$ (\Cref{eq:policy regret}) is
\begin{align*}
\mathcal L_t(\bm W_t)-\mathcal L_t(\bm W^\ast)&=\sum_{m=1}^M (\pi_m^\ast-\pi_m) r_m\overset{(a)}{=}\sum_{m=1}^M (\pi_m^\ast-\pi_m)r_m-\sum_{m=1}^M (\pi_m^\ast-\pi_m)\left (\sum_{m'=1}^M \pi_{m'} r_{m'}\right )\\
&\overset{(b)}{=}\sum_{m=1}^M \pi_m^\ast \left (r_m-\sum_{m'=1}^M \pi_{m'} r_{m'}\right )-\sum_{m=1}^M \pi_m r_m+\sum_{m'=1}^M \pi_{m'} r_{m'}\\
&=\sum_{m=1}^M \pi_m^\ast \left (r_m-\sum_{m'=1}^M \pi_{m'} r_{m'}\right ).
\end{align*}
where (a) uses the fact that $\sum_m (\pi_m^\ast-\pi_m) C=(1-1)C=0$ for any $C$ and (b) uses $\sum_{m} \pi_m=1$. 

For the linearized policy regret $\tilde \mR_T(\Pi_{\text{quad}})$ in \Cref{eq:linearized policy regret}, we have
\begin{align*}
\Bigl \langle \nabla_{\bm W}\mathcal L_t(\bm W_t),\bm W_t-\bm W^\ast\Bigr\rangle_F
&\overset{(a)}{=} \sum_{m=1}^M \left \langle -\sum_{m'=1}^M \pi_{m'} r_{m'} \bigl (\mathbbm 1[m'=m]-\pi_m\bigr) q_t q_t^\trans,W_{t,m}-W_{m}^\ast\right \rangle_F\\
&= \sum_{m=1}^M \left ( \pi_{m} r_{m}-\pi_m \sum_{m'=1}^M \pi_{m'} r_{m'}\right ) q_t^\trans\left (W_{m}^\ast-W_{t,m}\right )q_t \\
&\overset{(b)}{=}\sum_{m=1}^M \pi_m \left (r_m-\sum_{m'=1}^M \pi_{m'} r_{m'}\right ) \log \frac{\pi_m^\ast}{\pi_m},
\end{align*}
where (a) uses the direct sum property in \Cref{lem:direct sum} (proved in \Cref{sec:direct sum appendix}) and $\nabla \mathcal L_t(\bm W)$ derived in \Cref{eq:true gradient of loss} (rewritten using the $\pi_m$ and $r_m$ notations), and (b) uses the definition of $\pi_{\bm W}$:
\begin{align*}
&\quad \log \frac{\pi_{\bm W^\ast}(m\mid q_t)}{\pi_{\bm W_t}(m\mid q_t)}=\log \left (\frac{\exp(q_t^\trans W_m^\ast q_t)}{\sum_{m'=1}^M \exp(q_t^\trans W_{m'}^\ast q_t)}\right )\Bigg /\left (\frac{\exp(q_t^\trans W_{t,m} q_t)}{\sum_{m'=1}^M \exp(q_t^\trans W_{t,m'} q_t)}\right )\\
&=\log \frac{\exp(q_t^\trans W_m^\ast q_t)}{\exp(q_t^\trans W_{t,m} q_t)}+\log \frac{\sum_{m'=1}^M \exp(q_t^\trans W_{m'}^\ast q_t)}{\sum_{m'=1}^M \exp(q_t^\trans W_{t,m'} q_t)}=q_t^\trans (W_m^\ast-W_{t,m})q_t+C,
\end{align*}
where $C$ is (another) constant independent of $m$ that automatically vanishes because
\begin{equation*}
C\sum_{m=1}^M  \pi_m\left (r_m - \sum_{m'=1}^M \pi_{m'} r_{m'}\right )=C\sum_{m=1}^M \pi_m r_m-C\left (\sum_{m=1}^M \pi_m\right )\sum_{m'=1}^M\pi_{m'}r_{m'}=0.\qedhere
\end{equation*}
\end{proof}

\section{Setup of \Cref{tab:esci_direct_label_gaps} and More Discussions}\label{sec:esci_direct_label_gaps}

As mentioned in the main text, we first encode every natural language query and item into 768-dimensional vector spaces via \textsc{Contriever} \citep{izacard2022unsupervised}. For each embedded query $q \in \mathbb{R}^{768}$ and associated item $a \in \mathbb{R}^{768}$, the ESCI dataset labels it with one of four relevance scores: E (Exact), S (Substitute), C (Complement), or I (Irrelevant).
We convert the four ordinal ESCI labels into numerical rewards as
$E=0.95$, $S=0.70$, $C=0.30$, and $I=0.05$ (i.e., the reward function $r(q,a)$ in \Cref{eq:true reward}).
The specific values are a modeling choice that preserves the ordering of the labels and treats non-exact but related products (i.e., categories S and C) as partially useful.

We consider $M=8$ embedding models as candidate experts, spanning three distinct methodologies: \textit{(i)} symmetric semantic encoders capturing similarities (\texttt{all-MiniLM-L6}, \texttt{-L12}, and \texttt{paraphrase-MiniLM}; \citealt{wang2020minilm,wang2021minilmv2}); \textit{(ii)} asymmetric search models trained on Q\&A datasets (\texttt{multi-qa-MiniLM}, \texttt{msmarco-MiniLM}, and \texttt{msmarco-distilbert}; \citealt{reimers2019sentence,bajaj2016ms}); and \textit{(iii)} modern prompt-driven encoders dynamically adapting to different tasks (\texttt{e5-small-v2} and \texttt{bge-small-en-v1.5}; \citealt{wang2022text,xiao2024c}).

The recommendation distribution of each model $m\in [M]$ on a query $q\in \mQ$, namely the $\xi_m(q)$ in \Cref{eq:expert reward} but instead supported on the candidate items corresponding to this query, is the softmax over the query-item similarity scores induced by this model.
We then calculate the model's expected reward $R_m(q)=\E_{a\sim \xi_m(q)}[r(q,a)]$ and find the optimal model routing policies within each class:\footnote{Note that instead of online learning the optimal policies within different classes, in \Cref{tab:esci_direct_label_gaps} we perform offline optimization in order to verify the structural alignment of each class.}
\begin{itemize}
\item $\pi_\text{const}^\ast\in \Pi_{\text{const}}$ maps each query to the best fixed model $m^\ast=\argmax_m \E_q[R_m(q)]$, where the expectation is taken over the uniform distribution over all queries in the ESCI dataset;
\item $\pi_{\text{lin}}^\ast\in \Pi_{\text{lin}}$ is the best log-linear routing policy, i.e., in the form of $\pi(m\mid q)\propto \exp(\theta_m^\trans q)$ for some $\bm \theta=(\theta_1,\theta_2,\ldots,\theta_M)\in (\mathbb R^{768})^M$;
\item $\pi_{\text{quad}}^\ast\in \Pi_{\text{quad}}$ is the best log-quadratic routing policy, i.e., in the form of $\pi(m\mid q)\propto \exp(q^\trans W_m q)$ for some $\bm W=(W_1,W_2,\ldots,W_M)\in (\mathbb R^{768\times 768})^M$; and
\item $\pi^\ast\colon \mathcal Q\to \triangle$ is the best unrestricted policy, mapping any query $q$ to the embedding model (out of $M=8$) that performs the best on this specific query $q$. This, by definition, is the best model routing policy that one can possibly hope for.
\end{itemize}

In \Cref{tab:esci_direct_label_gaps}, we report the expected reward gap of each policy compared to $\pi^\ast$, that is,
\begin{equation}\label{eq:ESCI gap definition}
\text{Gap}(\pi):=\E_{q} \Bigl [\E_{m\sim \pi^\ast(q)}\bigl[R_m(q)\bigr]-\E_{m\sim \pi(q)}\bigl[R_m(q)\bigr]\Bigr ],\quad \forall \pi \in \{\pi_{\text{const}}^\ast,\pi_{\text{lin}}^\ast,\pi_{\text{quad}}^\ast\},
\end{equation}
where the outer expectation is again taken over the uniform distribution over all queries in ESCI.
We also report the relative improvement of $\pi_{\text{lin}}^\ast$ over $\pi_{\text{const}}^\ast$ and that of $\pi_{\text{quad}}^\ast$ over $\pi_{\text{lin}}^\ast$.
We see that $\text{Gap}(\pi_{\text{const}}^\ast)=0.035$, $\text{Gap}(\pi_{\text{lin}}^\ast)=0.030$, and $\text{Gap}(\pi_{\text{quad}}^\ast)=0.011$ (which, are all pretty small because the expert embedding models we chose are quite powerful).
We therefore observe that the improvement of $\pi_{\text{lin}}^\ast$ over $\pi_{\text{const}}^\ast$ is incremental (13.8\%), whereas our proposed log-quadratic policy $\pi_{\text{quad}}^\ast$ yields a $63.5\%$ improvement over $\pi_{\text{lin}}^\ast$ (and, $68.5\%$ over $\pi_{\text{const}}^\ast$). This indicates that our analysis \Cref{thm:quadratic policy class} is robust to the two discrepancies arising in the realistic ESCI dataset: \textit{(i)} the reward function $r(q,a)$ is not generated according to the bilinear structure in \Cref{eq:true reward}, and \textit{(ii)} the misspecification in recommendation model distributions -- in \Cref{eq:expert reward} -- can be large.

The ESCI dataset additionally includes several categories of challenging queries, where different embedding models are expected to exhibit different failure modes, making model routing more important. For example, \texttt{nlqec} consists of long natural-language search queries \citep{papenmeier2021dataset}. On this category, the best constant policy incurs an expected sub-optimality gap of $0.110$ (where the expectation is taken uniformly over all \texttt{nlqec} queries), larger than the $0.035$ gap over the full ESCI dataset; hence this category is indeed harder for any fixed expert model. The log-linear class reduces the gap to $0.004$, closing $96.0\%$ of the constant-policy gap, while our log-quadratic class nearly closes the gap entirely by another $99.7\%$ improvement.

Another category is \texttt{negations}, which contains negation words such as ``without.'' This category is not only hard for constant policies, but also hard for log-linear policies: the best $\pi_{\text{lin}}^{\texttt{negations}}\in \Pi_{\text{lin}}$ still incurs a sub-optimality gap of $0.054$, whereas our log-quadratic policy reduces it to $0.011$, an $80.1\%$ reduction relative to the log-linear gap. Similar results hold for \texttt{behavioral}, where such queries' resulting clicks or purchases have non-representative distributions, and \texttt{parse-pattern}, where queries exhibit linguistic complexity such as quantities or multiple modifiers.

\section{Omitted Proofs for Hypentropy Policy Gradient}

\subsection{Matrix Functions and Convex Analysis}\label{sec:main theorem appendix notations}
We first state a few properties of matrix functions and convex analysis without proof, mainly from \citet{kakade2012regularization}. Since our matrices are always symmetric, we restrict our definitions to the symmetric matrix space $\mathbb S^n$, though most of them can be seamlessly extended to rectangular matrices. 

\paragraph{Convex Analysis.}
We view $\mathbb S^n$ as a $n^2$-dimensional vector space equipped with the inner product:
\begin{equation*}
\langle X,Y\rangle_F=\tr(X^\trans Y)=\sum_{i=1}^n \sum_{j=1}^n X_{i,j}Y_{i,j},\quad \forall X,Y\in \mathbb S^n.
\end{equation*}
Over a convex domain $\mathcal K\subseteq \mathbb S^n$, a function $f\colon \mathcal K\to \mathbb R$ is $\alpha$-strongly convex w.r.t. some norm $\lVert \cdot\rVert$ if
\begin{equation}\label{eq:strongly convex}
f(X)-f(Y)-\nabla f(Y)~(X-Y)\ge \frac \alpha 2\lVert X-Y\rVert^2,\quad \forall X,Y\in \mathbb S^n.
\end{equation}
When $\alpha=0$, \Cref{eq:strongly convex} recovers the standard definition of convexity.
For a function $f\colon \mathcal K\to \mathbb R$ that is only (strongly) convex over a sub-domain $\mathcal K\subseteq \mathbb S^n$, we equivalently define $\bar f\colon \mathbb S^n\to \mathbb R\cup \{\infty\}$ as $\bar f(X)=\begin{cases}
f(X),&X\in \mathcal K\\
+\infty,&X\not\in \mathcal K
\end{cases}$, which is (strongly) convex over $\mathbb S^n$. Its Fenchel dual is defined as
\begin{equation*}
f^\ast(Z):=\sup_{X\in \mathbb S^n} \Bigl (\langle X,Z\rangle_F-\bar f(X)\Bigr ),\quad \forall Z\in \mathbb S^n.
\end{equation*}

\paragraph{Matrix Functions.}
For any symmetric real matrix $X\in \mathbb S^n$, it can always be eigen-decomposed as
\begin{equation*}
X=U\diag(\lambda(X))U^\trans,\quad \lambda(X):=[\lambda_1,\lambda_2,\ldots,\lambda_n],
\end{equation*}
with $\lambda_1\ge \lambda_2\ge \cdots \ge \lambda_n$ being the eigenvalues of $X$ and $U$ being an orthogonal matrix. The following trace inequality will be used in analyzing the computational complexity (\Cref{sec:efficient implementation appendix}):
\begin{lemma}[Trace Inequality {\citep{von1937some,fan1949theorem}}]\label{lem:vnm trace}
For $X,Y\in \mathbb S^n$, we have $\langle X,Y\rangle_F\le \langle \lambda(X),\lambda(Y)\rangle$. The equality holds if and only if there exists an orthogonal matrix $U$ such that
\begin{equation*}
X=U\diag(\lambda(X))U^\trans,\quad Y=U\diag(\lambda(Y))U^\trans.
\end{equation*}
\end{lemma}

A vector function $f\colon \mathbb R^n\to \mathbb R\cup \{\infty\}$ is symmetric if $f(x)$ is invariant under any permutations of the components of $x$. Any such function can be lifted to $\mathbb S^n$ as a $F\colon \mathbb S^n\to \mathbb R\cup \{\infty\}$, defined as
\begin{equation*}
F(X):=f(\lambda(X)),\quad \forall X=U\diag(\lambda(X))U^\trans\in \mathbb S^n.
\end{equation*}

\begin{lemma}[\citealt{kakade2012regularization}, Theorem 28]\label{lem:dual matrix func}
Let $F\colon \mathbb S^n\to \mathbb R\cup \{\infty\}$ be lifted from a symmetric $f\colon \mathbb R^n\to \mathbb R\cup \{\infty\}$.
The Fenchel dual of $F$, namely $F^\ast$, is equal to the Fenchel dual of $f$ lifted to $\mathbb S^n$.
\end{lemma}

\begin{lemma}[\citealt{kakade2012regularization}, Theorem 30]\label{lem:gradient matrix func}
Let $F\colon \mathbb S^n\to \mathbb R\cup \{\infty\}$ be lifted from a symmetric $f\colon \mathbb R^n\to \mathbb R\cup \{\infty\}$. The gradient of $F$ is given as
\begin{equation*}
\nabla F(X)=U \diag\bigl (\nabla f(\lambda(X))\bigr )U^\trans,\quad \forall X=U\diag(\lambda(X))U^\trans\in \mathbb S^n.
\end{equation*}
\end{lemma}
After defining gradients, we can define the Bregman divergence for any $F\colon \mathbb S^n\to \mathbb R\cup \{\infty\}$:
\begin{equation}\label{eq:Bregman divergence}
D_F(X\|Y):=F(X)-F(Y)-\langle \nabla F(Y),X-F\rangle_F,\quad \forall X,Y\in \mathbb S^n.
\end{equation}

\paragraph{Hypentropy Properties.}
We now move on to the $\beta$-Hypentropy function $\Phi_\beta\colon \mathbb S^n\to \mathbb R$ used in our HPG algorithm, which is defined as \citep{ghai2020exponentiated}:
\begin{equation*}
\Phi_\beta(X):=\sum_{i=1}^n \left (\lambda_i\arcsinh \frac{\lambda_i}{\beta}-\sqrt{\lambda_i^2+\beta^2}\right ),\quad \forall X\in \mathbb S^n,
\end{equation*}
where $\arcsinh(x)=\ln(x+\sqrt{x^2+1})$. It is lifted from the following vector function $\phi_\beta\colon \mathbb R^n\to \mathbb R$:
\begin{equation*}
\phi_\beta(x):=\sum_{i=1}^n \left (x_i\arcsinh \frac{x_i}{\beta}-\sqrt{x_i^2+\beta^2}\right ),\quad \forall x\in \mathbb R^n.
\end{equation*}

We have $\frac{\partial}{\partial x_i} \phi_\beta(x)=\arcsinh \frac{x_i}{\beta}$. According to \Cref{lem:gradient matrix func}, for any $X\in \mathbb S^n$, we have
\begin{equation}\label{eq:hypentropy gradient}
\nabla \Phi_\beta(X)=U\diag\left (\arcsinh \frac{\lambda(X)}{\beta}\right )U^\trans,\quad \forall X=U\diag(\lambda(X)) U^\trans \in \mathbb S^n.
\end{equation}
Plugging \Cref{eq:hypentropy gradient} into \Cref{eq:Bregman divergence} defines the Bregman divergence induced by $\Phi_\beta$, namely $D_\Phi^\beta$.

Finally, we include the following lemma, which ensures the hypentropy in $\mathbb S^n$ is strongly convex w.r.t. the nuclear norm $\lVert \cdot\rVert_\ast$ over any nuclear-norm ball.
\begin{lemma}[{\citealt{ghai2020exponentiated}, Theorem 14}]\label{lem:strongly convex}
On the nuclear-norm ball $\{W\in \mathbb S^n\mid \lVert W\rVert_\ast\le \tau\}$, $\Phi_\beta$ is $\bigl (2(\tau+\beta n)\bigr )^{-1}$-strongly convex w.r.t. the nuclear norm $\lVert \cdot\rVert_\ast$. That is,
\begin{equation*}
\Phi_\beta(X)-\Phi_\beta(Y)-\nabla \Phi_\beta(Y)~(X-Y)\ge \frac{\bigl (2(\tau+\beta n)\bigr )^{-1}}{2} \lVert X-Y\rVert_\ast^2,~~\forall X,Y\in \mathbb B_\ast(\tau).
\end{equation*}
\end{lemma}

\subsection{Direct Sum Property}\label{sec:direct sum appendix}
\begin{lemma}[Direct Sum Property]\label{lem:direct sum}
For any diagonal block matrix $\bm W=\diag(W_1,W_2,\ldots,W_M)$,
\begin{equation*}
\Phi_\beta(\bm W)=\sum_{m=1}^M \Phi_\beta(W_m),~\lVert \bm W\rVert_\ast=\sum_{m=1}^M \lVert W_m\rVert_\ast,~\lVert \bm W\rVert_2=\max_{m\in [M]} \lVert W_m\rVert_2,
\end{equation*}
which implies $D_\Phi^\beta(\bm W\|\bm W')=\sum_{m=1}^M D_\Phi^\beta(W_m\|W_m')$. Moreover, over the product nuclear-norm ball $\mathbb B_\ast^M(\tau):=\bigl\{\bm W\in \mW \mathrel{\big \vert} \lVert W_m\rVert_\ast\le \tau,\forall m\bigr \}$, $\Phi_\beta$ is block-wise $\bigl (2(\tau+\beta d_q)\bigr )^{-1}$-strongly convex.
\end{lemma}
\begin{proof}
Let $\lambda(X)$ denote the multiset of eigenvalues of a symmetric matrix $X$. Since $\bm W = \diag(W_1, W_2, \dots, W_M)\in \mW$ is a block-diagonal symmetric matrix, its eigenvalues are exactly the union of the eigenvalues of its diagonal blocks, i.e., $\lambda(\bm W) = \bigcup_{m=1}^M \lambda(W_m)$.

Because the nuclear norm $\lVert \cdot \rVert_\ast$, the spectral norm $\lVert \cdot \rVert_2$, and the hypentropy potential $\Phi_\beta$ are all spectral functions (i.e., defined w.r.t. eigenvalues), they can be decomposed into blocks:
\begin{align*}
\lVert \bm W\rVert_\ast &= \sum_{\lambda \in \lambda(\bm W)} |\lambda| = \sum_{m=1}^M \sum_{\lambda \in \lambda(W_m)} |\lambda| = \sum_{m=1}^M \lVert W_m\rVert_\ast, \\
\lVert \bm W\rVert_2 &= \max_{\lambda \in \lambda(\bm W)} |\lambda| = \max_{m\in [M]} \max_{\lambda \in \lambda(W_m)} |\lambda| = \max_{m\in [M]} \lVert W_m\rVert_2, \\
\Phi_\beta(\bm W) &= \sum_{\lambda \in \lambda(\bm W)} (\lambda+\beta)\log(\lambda+\beta) = \sum_{m=1}^M \sum_{\lambda \in \lambda(W_m)} (\lambda+\beta)\log(\lambda+\beta) = \sum_{m=1}^M \Phi_\beta(W_m).
\end{align*}

From \Cref{lem:gradient matrix func}, the gradient of a spectral function evaluated at a block-diagonal matrix remains block-diagonal, i.e., $\nabla \Phi_\beta(\bm W') = \diag\bigl(\nabla \Phi_\beta(W_1'), \dots, \nabla \Phi_\beta(W_M')\bigr)$. Further using the additivity of $\Phi_\beta$ and the fact that the Frobenius inner product also decomposes into the sum of block-wise inner products, the Bregman divergence decomposes as:
\begin{align*}
D_\Phi^\beta(\bm W\|\bm W') &= \Phi_\beta(\bm W) - \Phi_\beta(\bm W') - \langle \nabla \Phi_\beta(\bm W'), \bm W - \bm W' \rangle_F \\
&= \sum_{m=1}^M \left( \Phi_\beta(W_m) - \Phi_\beta(W_m') - \langle \nabla \Phi_\beta(W_m'), W_m - W_m' \rangle_F \right) \\
&= \sum_{m=1}^M D_\Phi^\beta(W_m\|W_m').
\end{align*}

From \Cref{lem:strongly convex}, over the $d_q$-dimensional nuclear-norm ball $\mathbb B_\ast(\tau)$, the hypentropy potential $\Phi_\beta$ is $\bigl (2(\tau+\beta d_q)\bigr )^{-1}$-strongly convex. For any $\bm W \in \mathbb B_\ast^M(\tau)$, we have $W_m \in \mathbb B_\ast(\tau)$ for all $m \in [M]$. Since $\Phi_\beta(\bm W) = \sum_{m=1}^M \Phi_\beta(W_m)$, the Hessian $\nabla^2 \Phi_\beta(\bm W)$ is block-diagonal where each block $\nabla^2 \Phi_\beta(W_m)$ satisfies $\alpha$-strong convexity. Thus, $\Phi_\beta$ is block-wise $\alpha$-strongly convex over $\mathbb B_\ast^M(\tau)$.
\end{proof}

\subsection{Gradient Estimation (\Cref{lem:REINFORCE})}\label{sec:REINFORCE appendix}

\begin{proof}[Proof of \Cref{lem:REINFORCE}]
Recall that the round-$t$ loss function for parameter $\bm W$, namely $\mathcal L_t(\bm W)$, is
\begin{equation*}
\mathcal L_t(\bm W)=-\E_{m'\sim \pi_{\bm W}(q_t)}\left [\E_{a\sim \xi_{m'}(q_t)}[r(q_t,a)]\right ],\quad \forall t\in [T],\bm W\in \mW,
\end{equation*}
where $\pi_{\bm W}$ is the log-quadratic policy induced by $\bm W$ (see \Cref{eq:quadratic policy}). We intentionally changed all $m$'s in the original definition to $m'$, so that we can take the partial derivative w.r.t. any $m\in [M]$ as:
\begin{equation*}
\frac{\partial}{\partial W_{m}}\mathcal L_t(\bm W)=-\sum_{m'=1}^M \left (\E_{a\sim \xi_{m'}(q_t)}[r(q_t,a)]~\frac{\partial}{\partial W_{m}} \pi_{\bm W}(m'\mid q_t)\right ).
\end{equation*}
By definition of $\pi_{\bm W}$ in \Cref{eq:quadratic policy}, we have by the definition of softmax that
\begin{equation*}
\frac{\partial}{\partial W_{m}} \pi_{\bm W}(m'\mid q_t)=\frac{\partial}{\partial W_{m}} \frac{\exp(q_t^\trans W_{m'} q_t)}{\sum_{m''} \exp(q_t^\trans W_{m''}q_t)}=\pi_{\bm W}(m'\mid q_t) \left (\mathbbm{1}[m'=m]-\pi_{\bm W}(m\mid q_t)\right ) q_tq_t^\trans.
\end{equation*}
Therefore, we have for any $m\in [M]$ that
\begin{equation}\label{eq:true gradient of loss}
\frac{\partial}{\partial W_{m}}\mathcal L_t(\bm W)=-\E_{m'\sim \pi_{\bm W}(q_t)} \left [\E_{a\sim \xi_{m'}(q_t)}[r(q_t,a)]~\left (\mathbbm{1}[m'=m]-\pi_{\bm W}(m\mid q_t)\right ) q_tq_t^\trans\right ].
\end{equation}

Hence conditional on the history before round $t$ and the round-$t$ query $q_t$, the expectation of $\hat G_{t,m}$ in \Cref{eq:REINFORCE} -- taken w.r.t. $m_t\sim \pi_{\bm W_t}(q_t)$, $a_t\sim \xi_{m_t}(q_t)$, and $r_t=r(q_t,a_t)+\eta_t$ -- is given as
\begin{align*}
&\quad \E\nolimits_t[\hat G_{t,m}\mid q_t]=\E\nolimits_t\Bigl [-r_t \bigl (\mathbbm{1}[m_t=m]-\pi_{\bm W_t}(m\mid q_t)\bigr )q_tq_t^\trans\mathrel{\Big \vert} q_t\Bigr ]\\
&=\E_{m_t\sim \pi_{\bm W_t}(q_t)}\left [\E_{a_t\sim \xi_{m_t}(q_t)} \left [\E_{\eta_t} \Bigl [-r_t \bigl (\mathbbm{1}[m_t=m]-\pi_{\bm W_t}(m\mid q_t)\bigr )q_tq_t^\trans\mathrel{\Big \vert} q_t,m_t,a_t \Bigr]\mathrel{\Big \vert} q_t,m_t\right ]\right ] \\
&=\E_{m_t\sim \pi_{\bm W_t}(q_t)}\left [\E_{a_t\sim \xi_{m_t}(q_t)} \left [-r(q_t,a_t) \bigl (\mathbbm{1}[m_t=m]-\pi_{\bm W_t}(m\mid q_t)\bigr )q_tq_t^\trans\mathrel{\Big \vert} q_t,m_t\right ]\right ] \\
&=\E_{m_t\sim \pi_{\bm W_t}(q_t)}\left [-\E_{a_t\sim \xi_{m_t}(q_t)} [r(q_t,a_t)]\bigl (\mathbbm{1}[m_t=m]-\pi_{\bm W_t}(m\mid q_t)\bigr )q_tq_t^\trans \right ] \\
&=\frac{\partial}{\partial W_{m}}\mathcal L_t(\bm W_t),\quad \forall m\in [M].
\end{align*}
Thus we have $\E_t[\hat{\bm G}_t\mid q_t]=\nabla_{\bm W} \mathcal L_t(\bm W_t)$ unbiased, as claimed. For the second-order moments,
\begin{align*}
\sum_{m=1}^M \lVert \hat G_{t,m}\rVert_{2}^2&=r_t^2 \lVert q_t\rVert_2^4 \sum_{m=1}^M \bigl (\mathbbm 1[m_t=m]-\pi_{\bm W_t}(m\mid q_t)\bigr )^2 \\&\le r_t^2\left ((1-\pi_{\bm W_t}(m_t\mid q_t))^2+\sum_{m=1}^M \pi_{\bm W_t}^2(m\mid q_t)\right )\le 2r_t^2,
\end{align*}
where the first inequality uses $\lVert q_tq_t^\trans\rVert_2=\lVert q_t\rVert_2^2\le 1$, and the second one uses the fact that $\pi_{\bm W_t}(q_t)\in \triangle$ is a probability distribution (hence $\sum_m \pi_{\bm W_t}^2(m\mid q_t)\le \sum_m \pi_{\bm W_t}(m\mid q_t)=1$).

Since $\lvert r(q,a)\rvert\le \lVert \Psi^\ast\rVert_2+\lVert \nu\rVert_\infty\le 1+\epsilon$ while $\eta_t$ is zero-mean and conditionally $1$-sub-Gaussian, we have
\begin{equation*}
\E\nolimits_t[r_t^2\mid q_t]\le (1+\epsilon)^2+1\le 5,\quad \text{and thus}\quad \E\nolimits_t\left [\sum_{m=1}^M \lVert \hat G_{t,m}\rVert_2^2\mathrel{\Big \vert} q_t\right ]\le 10.
\end{equation*}
\end{proof}

\subsection{Linearized Policy Regret Bound of HPG (\Cref{thm:main theorem})}\label{sec:main theorem appendix}
\begin{proof}[Proof of \Cref{thm:main theorem}]
According to the assumption that the optimal $\bm W^\ast\in \mW$ ensures $\lVert W_m^\ast\rVert_\ast\le \tau$, $\forall m$, we have $\bm W^\ast\in \mathbb B_\ast^M(\tau)$ (recall \Cref{eq:product nuclear norm ball}). Recall the OMD one-step update rule in \Cref{eq:OMD}:
\begin{equation*}
\bm W_{t+1}=\argmin_{\bm W\in \mathbb B_\ast^M(\tau)}\Bigl (\eta \bigl \langle \bm W,\hat{\bm G}_t\bigr \rangle_F +D_\Phi^\beta(\bm W\|\bm W_t)\Bigr ).
\end{equation*}
Using the direct sum property \Cref{lem:direct sum} later proved in \Cref{sec:direct sum appendix}, this is equivalent to
\begin{equation}\label{eq:OMD one step}
W_{t+1,m}=\argmin_{W\in \mathbb B_\ast(\tau)} \Bigl (\eta \bigl \langle W,\hat G_{t,m}\bigr \rangle_F +D_\Phi^\beta(W\|W_{t,m})\Bigr ),\quad \forall m\in [M],
\end{equation}
where we recall that $\mathbb B_\ast(\tau):=\{W\in \mathbb S^{d_q}\mid \lVert W\rVert_\ast\le \tau\}$ is the nuclear-norm ball in $\mathbb S^{d_q}$.
The rest of the proof follows from standard OMD analysis. Since $\nabla D_\Phi^\beta(W\|W_{t,m})=\nabla \Phi_\beta(W)-\nabla \Phi_\beta(W_{t,m})$ (by definition of $D_\Phi^\beta$ in \Cref{eq:Bregman divergence}), \Cref{eq:OMD one step} gives the first-order optimality condition of
\begin{equation*}
\Bigl \langle \eta \hat G_{t,m} + \nabla \Phi_\beta(W_{t+1,m})-\nabla \Phi_\beta(W_{t,m}),W-W_{t+1,m}\Bigr \rangle_F \ge 0,\quad \forall W\in \mathbb B_\ast(\tau).
\end{equation*}
More specifically, for the $m$-th parameter in the optimum, namely $W_m^\ast\in \mathbb B_\ast(\tau)$, we have
\begin{align*}
\langle \eta \hat G_{t,m},W_{t+1,m}-W_m^\ast\rangle_F&\le \Bigl \langle \nabla \Phi_\beta(W_{t+1,m})-\nabla \Phi_\beta(W_{t,m}),W_m^\ast-W_{t+1,m}\Bigr \rangle_F\\
&=D_\Phi^\beta(W_m^\ast\|W_{t,m})-D_\Phi^\beta(W_m^\ast\|W_{t+1,m})-D_\Phi^\beta(W_{t+1,m}\|W_{t,m}),
\end{align*}
where the second step is the standard three-point identity \citep[Lemma 6.7]{orabona2019modern}.
Adding $\langle \eta \hat{G}_{t,m},W_{t,m}-W_{t+1,m}\rangle_F$ on both sides and summing up from $t=1,2,\ldots,T$, we have
\begin{align*}
\sum_{t=1}^T \langle\hat G_{t,m},W_{t,m}-W_m^\ast\rangle_F&\le \frac{D_\Phi^\beta(W_m^\ast\|W_{1,m})-D_\Phi^\beta(W_m^\ast\|W_{T+1,m})}{\eta}\\&\quad +\sum_{t=1}^T \frac{\langle \eta \hat{G}_{t,m},W_{t,m}-W_{t+1,m}\rangle_F-D_\Phi^\beta(W_{t+1,m}\|W_{t,m})}{\eta}.
\end{align*}
For the first term on the RHS, since $W_{1,m}$ is initialized as the all-zero matrix (\Cref{alg:OMD}) and Bregman divergences are non-negative for convex functions (cf. \Cref{eq:strongly convex,eq:Bregman divergence}), we have
\begin{align*}
&\quad D_\Phi^\beta(W_m^\ast\|W_{1,m})-D_\Phi^\beta(W_m^\ast\|W_{T+1,m})\le \Phi_\beta(W_m^\ast)-\Phi_\beta(0)+0\\
&=\sum_{i=1}^{d_q} \left (\lambda_i(W_m^\ast)\arcsinh \frac{\lambda_i(W_m^\ast)}{\beta}-\sqrt{\lambda_i(W_m^\ast)^2+\beta^2}+\beta\right )\\
&\le \sum_{i=1}^{d_q} \lvert \lambda_i(W_m^\ast)\rvert \arcsinh \frac{\lvert \lambda_i(W_m^\ast)\rvert }{\beta}\le \lVert W_m^\ast\rVert_\ast \log \frac{3 \lVert W_m^\ast\rVert_\ast}{\beta},
\end{align*}
where the second to last step uses the fact that $\arcsinh$ is an odd function, and the last step uses $\sqrt{x^2+\beta^2}+x\le \sqrt 2+1$ for $\lvert x\rvert\le 1$ \citep[Eq. (12)]{ghai2020exponentiated} and $\lVert W_m^\ast\rVert_\ast:=\sum_{i=1}^{d_q} \lvert \lambda_i(W_m^\ast)\rvert$.

For the second term, by Fenchel-Young inequality \citep[Lemma 6.32]{orabona2019modern} and the block-wise $\bigl(2(\tau+\beta d_q)\bigr)^{-1}$-strong convexity of $\Phi_\beta$ (proved in \Cref{lem:direct sum}), we have
\begin{equation*}
\langle \eta \hat{G}_{t,m},W_{t,m}-W_{t+1,m}\rangle_F-D_\Phi^\beta(W_{t+1,m}\|W_{t,m})\le \frac{\lVert \eta \hat G_{t,m}\rVert_2^2}{2\bigl(2(\tau+\beta d_q)\bigr)^{-1}},
\end{equation*}
where $\lVert \cdot\rVert_2$, the spectral norm of a matrix, is the dual norm of the nuclear norm $\lVert \cdot\rVert_\ast$. Hence we have
\begin{equation*}
\sum_{t=1}^T \langle\hat G_{t,m},W_{t,m}-W_m^\ast\rangle_F\le \frac{\lVert W_m^\ast\rVert_\ast}{\eta} \log \frac{3 \lVert W_m^\ast\rVert_\ast}{\beta}+\sum_{t=1}^T (\tau+\beta d_q) \eta \lVert \hat G_{t,m}\rVert_2^2,\quad \forall m\in [M].
\end{equation*}
Summing up $m\in [M]$ and using the assumption that $\lVert W_m^\ast\rVert_\ast\le \tau$, we arrive at
\begin{equation}\label{eq:online learning regret bound}
\sum_{t=1}^T \Bigl \langle \hat{\bm G}_t,\bm W_t-\bm W^\ast\Bigr \rangle_F\le \frac{M\tau}\eta \log \frac{3\tau}{\beta} + (\tau+\beta d_q)\eta \sum_{t=1}^T \sum_{m=1}^M \lVert \hat G_{t,m}\rVert_2^2.
\end{equation}

Taking expectations on both sides and invoking the second part of \Cref{lem:REINFORCE} gives
\begin{equation*}
\E\left [\sum_{t=1}^T \Bigl \langle \hat{\bm G}_t,\bm W_t-\bm W^\ast\Bigr \rangle_F\right ]\le \frac{M\tau}\eta \log \frac{3\tau}{\beta} + 10 (\tau+\beta d_q)\eta T.
\end{equation*}

By the definition of linearized policy regret in \Cref{eq:linearized policy regret}, setting $\tau=2s$, $\beta=2s/d_q$, and $\eta=\sqrt{(M\log d_q) / T}$ therefore gives
\begin{equation*}
\tilde \mR_T(\Pi_{\text{quad}})=\sup_{\bm W^\ast\in \mW} \E\left [\sum_{t=1}^T \Bigl \langle \hat{\bm G}_t,\bm W_t-\bm W^\ast\Bigr \rangle_F\right ]\le 44s \sqrt{MT \log d_q},
\end{equation*}
where we used the first part of \Cref{lem:REINFORCE} and also the assumption that $\lVert W_m^\ast\rVert_\ast\le 2s$, $\forall m$.
\end{proof}

\section{Practical Implementation of HPG}
\subsection{Computational Cost Bound of HPG (\Cref{lem:efficient implementation})}\label{sec:efficient implementation appendix}

\begin{proof}[Proof of \Cref{lem:efficient implementation}]
By the standard property of Bregman divergence, the single-step update of OMD in \Cref{eq:OMD} is equivalent to the following two step procedure \citep[\S6.4.3]{orabona2019modern}:
\begin{align}\label{eq:OMD two step update}
\nabla \Phi_\beta(\tilde{\bm W}_{t+1})=\nabla \Phi_\beta(\bm W_t)-\eta \hat{\bm G}_t,\quad \bm W_{t+1}=\argmin_{\bm W\in \mathbb B_\ast^M(\tau)} D_\Phi^\beta(\bm W\|\tilde{\bm W}_{t+1}).
\end{align}
According to \Cref{lem:direct sum}, both steps can be decomposed for each model. Therefore, we fix a single $m\in [M]$, and suppose we have the eigen-decomposition of $W_{t,m}\in \mathbb S^{d_q}$ available:
\begin{equation}\label{eq:SVD of W_t appendix}
W_{t,m}=U_{t,m}\diag(\lambda_{t,m})U_{t,m}^\trans,\quad \forall t\in [T],
\end{equation}
where $U_{t,m}\in \mathbb R^{d_q\times d_q}$ and $\lambda_{t,m}\in \mathbb R^{d_q}$. We now prove that we can efficiently compute the eigen-decomposition of $W_{t+1,m}$, defined in \Cref{eq:OMD two step update}, via the following steps (including the two mirror-space conversions folded into the first step of \Cref{sec:efficient implementation}):
\begin{enumerate}
\item \textbf{Convert $\bm W_t$ to Mirror Space.} We find $Y_{t,m}=\nabla \Phi_\beta(W_{t,m})$. As defined in \Cref{eq:hypentropy gradient}, we have
\begin{equation}\label{eq:SVD of Y_t appendix}
Y_{t,m}=U_{t,m}\diag(\gamma_{t,m})U_{t,m}^\trans,\quad \gamma_{t,m,i}=\arcsinh(\lambda_{t,m,i}/\beta),~\forall i\in [d_q].
\end{equation}
This element-wise update on the $d_q$ eigenvalues takes $\O(d_q)$ time.

\item \textbf{Gradient Descent Update.} In the mirror space, we perform gradient descent $\tilde Y_{t+1,m}=Y_{t,m}-\eta \hat G_{t,m}$ and obtain the eigen-decomposition of $\tilde Y_{t+1,m}$. Since $\hat G_{t,m}$ is parallel to $q_tq_t^\trans$, it is a rank-one matrix (see \Cref{eq:REINFORCE}). Therefore, eigen-decomposing $\tilde Y_{t+1,m}$ based on that of $Y_{t,m}$ in \Cref{eq:SVD of Y_t appendix} only takes $\O(d_q^2)$ time \citep[see, e.g.,][]{bunch1978rank,gu1994stable}:
\begin{equation}\label{eq:SVD of tilde Y_t+1 appendix}
\tilde Y_{t+1,m}=U_{t+1,m}\diag(\tilde \gamma_{t+1,m})U_{t+1,m}^\trans,
\end{equation}
where $U_{t+1,m}\in \mathbb R^{d_q\times d_q}$ and $\tilde \gamma_{t+1,m}\in \mathbb R^{d_q}$.
\item \textbf{Convert $\tilde{\bm Y}_{t+1}$ to Primal Space.} Let $\tilde W_{t+1,m}=(\nabla \Phi_\beta)^{-1}(\tilde Y_{t+1,m})$. Using \Cref{lem:gradient matrix func} again, it is equivalent to applying the map $(\phi_\beta)^{-1}(y)=\beta \sinh(y)$ to all the eigenvalues of $\tilde Y_{t+1,m}$. Therefore, from \Cref{eq:SVD of tilde Y_t+1 appendix}, we obtain the eigen-decomposition of $\tilde W_{t+1,m}$ in $\O(d_q)$ time as
\begin{equation}\label{eq:SVD of tilde W_t+1 appendix}
\tilde W_{t+1,m}=U_{t+1,m}\diag(\tilde \lambda_{t+1,m})U_{t+1,m}^\trans,\quad\tilde \lambda_{t+1,m,i}=\beta \sinh(\tilde \gamma_{t+1,m,i}),~\forall i\in [d_q].
\end{equation}

\item \textbf{Projection onto $\mathbb B_\ast(\tau)$ based on $D_\Phi^\beta$.} Expanding the Bregman divergence defined in \Cref{eq:Bregman divergence}, the projection step in \Cref{eq:OMD two step update} minimizes
\begin{equation*}
D_\Phi^\beta(W\|\tilde W_{t+1,m})=\Phi_\beta(W)-\Phi_\beta(\tilde W_{t+1,m})+\bigl \langle \nabla \Phi_\beta(\tilde W_{t+1,m}),W-\tilde W_{t+1,m}\bigr \rangle_F.
\end{equation*}
Removing all terms independent to $W$ and recalling that $\tilde Y_{t+1,m}=\nabla \Phi_\beta(\tilde W_{t+1,m})$, we write
\begin{align*}
&\quad W_{t+1,m}=\argmin_{W\in \mathbb B_\ast(\tau)}\left (\Phi_\beta(W)-\bigl \langle \tilde Y_{t+1,m},W\bigr\rangle_F\right )\\
&=\argmin_{\lVert \lambda\rVert_1\le \tau} \argmin_{U^\trans U=I} \left (\sum_{i=1}^{d_q} \phi_\beta(\lambda_i)-\bigl\langle U_{t+1,m} \diag(\tilde \gamma_{t+1,m}) U_{t+1,m}^\trans,U\diag(\lambda)U^\trans \bigr\rangle_F\right ),
\end{align*}
where we used the fact that any symmetric matrix $W\in \mathbb B_\ast(\tau)$ can be written as $U\diag(\lambda)U^\trans$ with $\lVert \lambda\rVert_1\le \tau$ (definition of nuclear norm) and $UU^\trans=I$ (definition of orthogonal matrix).

Fixing $\lambda\in \mathbb R^{d_q}$, apply the trace inequality in \Cref{lem:vnm trace} to the second term: the inner product is at most $\langle \tilde \gamma_{t+1,m},\lambda\rangle$, with equality attained if and only if $U=U_{t+1,m}$.
Therefore, we find
\begin{equation*}
\argmin_{\lVert \lambda\rVert_1\le \tau} \left (\sum_{t=1}^{d_q} (\phi_\beta(\lambda_i)-\tilde \gamma_{t+1,m,i} \lambda_i)\right ).
\end{equation*}
For each $i\in [d_q]$, minimizing this is equivalent to minimizing the Bregman divergence induced by $\phi_\beta$ (again, writing out the definition and removing all terms independent to $\lambda_i$). This gives
\begin{equation}\label{eq:SVD of W_t+1 appendix}
W_{t+1,m}=U_{t+1,m}\diag(\lambda_{t+1,m})U_{t+1,m}^\trans,\quad \lambda_{t+1,m}=\argmin_{\lVert \lambda\rVert_1\le \tau} \sum_{i=1}^{d_q} D_\phi^\beta(\lambda_i\|\tilde \lambda_{t+1,m,i}),
\end{equation}
as claimed in the main text (\Cref{eq:SVD of W_t+1}). This is implementable in $\O(d_q\log d_q)$ time as follows:
By symmetry, the optimal $\lambda_i$ must share the same sign as $\tilde \gamma_{t+1,m,i}$. Let $x_i = |\lambda_i|$ and $z_i = |\tilde \gamma_{t+1,m,i}|$,
\begin{equation*}
\min_{x_i \ge 0} \sum_{i=1}^{d_q} \left( \phi_\beta(x_i) - z_i x_i \right)\text{ subject to }\sum_{i=1}^{d_q} x_i \le \tau.
\end{equation*}
Introducing a Lagrange multiplier $\nu \ge 0$ for the sum constraint. Given any multiplier $\nu$,
\begin{equation}\label{eq:low-rank from Lagra ngian}
\phi_\beta'(x_i)-z_i+\nu\ge 0 \Longrightarrow x_i(\nu) = \beta \max \Bigl (\sinh(z_i - \nu),0\Bigr),
\end{equation}
and the optimal $\nu^\ast$ is the unique non-negative number satisfying $\sum_{i=1}^{d_q} x_i(\nu) \le \tau$. Since $\sinh$ is monotone, $\nu^\ast$ is found after sorting $\{z_i\}_i$ in $\O(d_q\log d_q)$ time. This gives $x_i(\nu^\ast)$'s and \Cref{eq:SVD of W_t+1 appendix}.
\end{enumerate}

Therefore, we have proved that moving from the eigen-decomposition of $W_{t,m}$ in \Cref{eq:SVD of W_t appendix} to that of $W_{t+1,m}$ in \Cref{eq:SVD of W_t+1 appendix} takes $\O(d_q+d_q^2+d_q+d_q\log d_q)=\O(d_q^2)$ time. Since there are $M$ models to update in each round, this gives the claimed per-round complexity bound of $\O(d_q^2M)$.
\end{proof}

In practice, one can sharpen the $\O(d_q^2 M)$ bound via approximation: \Cref{eq:low-rank from Lagra ngian} induces a low-rank structure, because the sum of $x_i(\nu^\ast)$'s cannot exceed $\tau$. Therefore, a good approximation of HPG algorithm is only keeping the top-$\tau$ eigenvalues of $W_{t+1,m}$. This makes $\lambda_{t+1,m}$ $\tau$-dimensional and $U_{t+1,m}$ in the shape of $d_q\times \tau$. The rank-one eigen-update in \Cref{eq:SVD of tilde Y_t+1 appendix} is then solvable in $\O(\tau d_q)$ time. Under the $\tau\approx s$ configuration in \Cref{thm:main theorem}, this only takes $\O(sd_q M)$ time per round.

\subsection{Parameter-Free Implementation of HPG (\Cref{thm:parameter-free})}\label{sec:parameter-free appendix}

\begin{algorithm}[!t]
\caption{Parameter-Free Hypentropy Policy Gradient}\label{alg:parameter-free}
\begin{algorithmic}[1]
\State Initialize $M$ 1D ``coin-betting'' algorithms \citep[Algorithm 1]{cutkosky2018black}, where the action sets are $[0,Z]$ with $Z:=2d_q$
\State Initialize a HPG algorithm with $\tau=1,\beta=d_q^{-1}$, and $\eta=\sqrt{(M\log d_q)/T}$ (\Cref{alg:OMD})
\For{$t=1,2,\ldots,T$}
\State Call each coin-betting algorithm for magnitudes $z_{t,1},z_{t,2},\ldots,z_{t,M}\in [0,Z]$
\State Call HPG algorithm for a direction $\bm w_t=\diag(w_{t,1},w_{t,2},\ldots,w_{t,M})\in \mathbb B_\ast^M(1)$
\State Play according to the parameter $\bm W_t=\diag(z_{t,1} w_{t,1},z_{t,2} w_{t,2},\ldots,z_{t,M}w_{t,M})\in \mW$
\State Clip the observed reward as $\bar r_t=\min\bigl (\max(r_t,-R),R\bigr )$, where $R:=2+\sqrt{2\log (2T^2)}$
\State Let $\hat{\bm G}_t=\diag(\hat G_{t,1},\hat G_{t,2},\ldots,\hat G_{t,M})$ be \Cref{eq:REINFORCE} evaluated at $\bar r_t$
\State For each $m\in [M]$, pass $\langle \hat G_{t,m},w_{t,m}\rangle_F/R$ to the $m$-th coin-betting algorithm
\State Pass $\hat{\bm G}_t$ to the HPG algorithm as the gradient for the OMD step \Cref{eq:OMD}
\EndFor
\end{algorithmic}
\end{algorithm}

In order to implement the HPG algorithm without knowing the low-rank parameter $s$ or the nuclear-norm parameter $\tau=\max_m \lVert W_m^\ast\rVert_\ast$, we utilize the parameter-free online learning technique by \citet{cutkosky2018black}. They provide a general framework for online linear optimization in any Banach space (a vector space with norm, e.g., our $\mathbb S^n$ equipped with the nuclear norm).

In our context, given the direct sum property in \Cref{lem:direct sum}, we consider a single expert $m\in [M]$. We now work on the $d_q^2$-dimensional vector space equipped with nuclear norm $\lVert \cdot\rVert_\ast$. The action in round $t$ is $W_{t,m}\in \mathbb S^{d_q}$, and the loss vector handed to the algorithm is the gradient estimator $\hat G_{t,m}$ of \Cref{eq:REINFORCE}, not the true gradient $\frac{\partial}{\partial W_m}\mathcal L_t(\bm W_t)$. Every bound below is therefore pathwise in $\hat{\bm G}_t$, and we pass to the expected $\tilde \mR_T(\Pi_{\text{quad}})$ of \Cref{eq:linearized policy regret} only at the end.

\citet{cutkosky2018black} decompose $W_{t,m}$ into two parts: a magnitude $z_{t,m}=\lVert W_{t,m}\rVert_\ast$ and a direction $w_{t,m}=\frac{W_{t,m}}{\lVert  W_{t,m}\rVert_\ast}$. The decision of $z_{t,m}$ is a one-dimensional decision-making problem over $\mathbb R$, and can be done via the \emph{coin-betting} algorithm of \citet{orabona2016coin}. For the direction $w_{t,m}$, which always has norm 1, our HPG algorithm over $\mathbb B_\ast(1)$ attains good regret. This gives \Cref{alg:parameter-free}.

The coin-betting subroutine requires losses bounded by $1$ a.s. but our reward $r_t$ is unbounded due to the sub-Gaussian noise $\eta_t$. We adopt a standard clipping argument: We clip each $r_t$ to $\bar r_t\in [-R,R]$ with $R:=2+\sqrt{2\log(2T^2)}$, and thus the coin-betting loss $\langle \hat G_{t,m},w_{t,m}\rangle_F$ lies in $[-R,R]$ a.s. While this makes $\hat{\bm G}_t$ biased, the bias can be controlled once $z_{t,m}=\lVert W_{t,m}\rVert_\ast$'s are bounded (see \Cref{eq:parameter-free clipping error decomposition} in the proof). We therefore restrict each magnitude $z_{t,m}$ to $[0,Z]$ with $Z:=2d_q$ (this ensures that $\bm W^\ast$, satisfying $\lVert W_m^\ast\rVert_\ast\le 2s\le Z$, is not excluded), using the constraint set reduction of \citet[Algorithm 3]{cutkosky2018black}. Their Theorem 3 shows that this at most doubles the regret.
\begin{proof}[Proof of \Cref{thm:parameter-free}]
We first write out the clipping error. Let $\bm \Delta_t$ be \Cref{eq:REINFORCE} evaluated at $r_t$ minus the same expression at $\bar r_t$ (which is the $\hat{\bm G}_t$ in \Cref{alg:parameter-free}). Since $\bm W_t$ is decided before round $t$'s randomness, the first half of \Cref{lem:REINFORCE} gives $\E\nolimits_t[\hat{\bm G}_t+\bm \Delta_t\mid q_t]=\nabla \mathcal L_t(\bm W_t)$. Consequently,
\begin{equation}\label{eq:parameter-free clipping error decomposition}
\tilde \mR_T(\Pi_{\text{quad}})=\E\left [\sum_{t=1}^T \bigl \langle \hat{\bm G}_t,\bm W_t-\bm W^\ast\bigr \rangle_F\right ]+\E\left [\sum_{t=1}^T \bigl \langle \bm \Delta_t,\bm W_t-\bm W^\ast\bigr \rangle_F\right ],
\end{equation}
where $\bm W^\ast$ is the optimal parameter $\bm W^\ast\in \mW$ that we assume to satisfy $\lVert W_m^\ast\rVert_\ast\le 2s$, $\forall m$. Focusing on the parameter-free online learning error $\E [\sum_{t=1}^T  \langle \hat{\bm G}_t,\bm W_t-\bm W^\ast \rangle_F ]$, we have for all $m$:
\begin{align*}
&\quad \sum_{t=1}^T \langle \hat G_{t,m},W_{t,m}-W_{m}^\ast\rangle_F=\sum_{t=1}^T \Bigl (\langle \hat G_{t,m},z_{t,m} w_{t,m}\rangle_F-\langle \hat G_{t,m},W_m^\ast\rangle_F\Bigr )\\
&=\sum_{t=1}^T \Bigl (z_{t,m}\langle \hat G_{t,m},w_{t,m}\rangle_F - \lVert W_m^\ast\rVert_\ast \langle \hat G_{t,m},w_{t,m}\rangle_F \Bigr )+\sum_{t=1}^T \Bigl (\lVert W_m^\ast\rVert_\ast \langle \hat G_{t,m},w_{t,m}\rangle_F-\langle \hat G_{t,m},W_m^\ast\rangle_F\Bigr ).
\end{align*}

The first term is the 1-dimensional online learning regret of $z_{t,m}$ w.r.t. the comparator $\lVert W_m^\ast\rVert_\ast\in [0,Z]$, with the round-$t$ loss $g_{t,m}:=\langle \hat G_{t,m}, w_{t,m}\rangle_F$. Clipping gives $\lvert g_{t,m}\rvert\le \lVert \hat G_{t,m}\rVert_2 \lVert w_{t,m}\rVert_\ast\le \lvert \bar r_t\rvert\le R$ (see \Cref{eq:REINFORCE}), so the rescaled losses $g_{t,m}/R$ fed to the subroutine lie in $[-1,1]$ a.s. The coin-betting guarantee \citep[Theorem 1]{cutkosky2018black} therefore applies pathwise as
\begin{align*}
\sum_{t=1}^T \bigl (z_{t,m} - \lVert W_m^\ast\rVert_\ast \bigr )g_{t,m}&\le \O\left (R\lVert W_m^\ast\rVert_\ast L_m+\lVert W_m^\ast\rVert_\ast \sqrt{L_m\sum_{t=1}^T g_{t,m}^2}\right ),~L_m:=\log \bigl (\lVert W_m^\ast\rVert_\ast T\bigr )\\
&\le \O\left (RsM \log (sT)+s \sqrt{M\log (sT)\sum_{t=1}^T\sum_{m=1}^M g_{t,m}^2}\right ).
\end{align*}

The second term, on the other hand, is $\lVert W_m^\ast\rVert_\ast$ times the HPG regret w.r.t. $\frac{W_m^\ast}{\lVert W_m^\ast\rVert_\ast}\in \mathbb B_\ast(1)$. Applying \Cref{eq:online learning regret bound} block-wise with $\tau=1$, $\beta=1/d_q$, and $\eta=\sqrt{(M\log d_q)/T}$, and using $\lVert W_m^\ast\rVert_\ast\le 2s$,
\begin{equation*}
\sum_{m=1}^M \sum_{t=1}^T \Bigl (\lVert W_m^\ast\rVert_\ast \langle \hat G_{t,m},w_{t,m}\rangle_F-\langle \hat G_{t,m},W_m^\ast\rangle_F\Bigr )\le 2s\left (\frac M\eta \log (3d_q)+2\eta \sum_{t=1}^T \sum_{m=1}^M \lVert \hat G_{t,m}\rVert_2^2\right ).
\end{equation*}

We now focus on the clipping error $\E[\sum_{t=1}^T \langle \bm \Delta_t,\bm W_t-\bm W^\ast\rangle_F ]$. We have $\lVert \Delta_{t,m}\rVert_2\le \lvert r_t-\bar r_t\rvert\le (\lvert \eta_t\rvert-u)_+$ for $u:=R-2=\sqrt{2\log (2T^2)}$, because $\lvert r(q,a)\rvert\le 1+\epsilon\le 2$. As $\eta_t$ is conditionally $1$-sub-Gaussian, integrating over the cases where $r_t\ne \bar r_t$ gives
\begin{equation*}
\E\nolimits_t\bigl [(\lvert \eta_t\rvert-u)_+\bigr ]=\int_u^\infty \Pr\bigl [\lvert \eta_t\rvert>v\bigr ]\,\mathrm dv\le \int_u^\infty 2e^{-v^2/2}\,\mathrm dv\le \frac 2ue^{-u^2/2}=\frac 1{uT^2}\le \frac 1{T^2}.
\end{equation*}
Because we have $\lVert W_{t,m}-W_m^\ast\rVert_\ast\le Z+2s\le 4d_q$ for all $m$, so the clipping error term is at most $T\cdot M\cdot T^{-2}\cdot 4d_q=\O(Md_q/T)=o(1)$.
We now collect the three terms and take expectations. Since $g_{t,m}^2\le \lVert \hat G_{t,m}\rVert_2^2$ and clipping only shrinks $\lvert r_t\rvert$, \Cref{lem:REINFORCE} gives $\E[\sum_{t}\sum_m g_{t,m}^2]\le 10T$. As $\sqrt\cdot$ is concave, Jensen's inequality turns the huge square root into $\sqrt{10MT}$. Setting $\eta=\sqrt{(M\log d_q)/T}$,
\begin{align*}
\tilde \mR_T(\Pi_{\text{quad}})&\le \O \Bigl (RsM\log (sT)+s\sqrt{MT\log (sT)}\Bigr )+\O\Bigl (s\sqrt{MT\log d_q}\Bigr )+o(1)\\&=\O \Bigl ( s\sqrt{MT\log (d_qT)}+sM\sqrt{\log T} \log(s T)\Bigr ),
\end{align*}
where the last step uses $R=\O(\sqrt{\log T})$ and $s\le d_q$. This proves \Cref{thm:parameter-free}.
\end{proof}

\end{document}